\documentclass{article}

\usepackage[preprint]{neurips_2026}

\usepackage[utf8]{inputenc}
\usepackage[T1]{fontenc}
\usepackage{hyperref}
\usepackage{url}
\usepackage{booktabs}
\usepackage{amsfonts}
\usepackage{nicefrac}
\usepackage{microtype}
\usepackage{xcolor}
\usepackage{graphicx}
\usepackage{bm}
\usepackage{amsmath, amssymb, amsthm, mathtools}
\usepackage{cleveref}
\usepackage{thmtools}
\usepackage{thm-restate}
\usepackage{newtxtext,newtxmath}
\usepackage{titlesec}

\titlespacing*{\section}{0pt}{1.5ex plus 0.5ex minus .2ex}{1ex plus .2ex}
\titlespacing*{\subsection}{0pt}{1.2ex plus .5ex minus .2ex}{0.8ex plus .2ex}

\newtheorem{theorem}{Theorem}[section]

\newtheorem{proposition}[theorem]{Proposition}
\newtheorem{corollary}[theorem]{Corollary}
\newtheorem{definition}[theorem]{Definition}
\newtheorem{remark}[theorem]{Remark}

\crefname{theorem}{Theorem}{Theorems}
\crefname{lemma}{Lemma}{Lemmas}
\crefname{proposition}{Proposition}{Propositions}
\crefname{corollary}{Corollary}{Corollaries}
\crefname{definition}{Definition}{Definitions}
\crefname{remark}{Remark}{Remarks}
\crefname{assumption}{Assumption}{Assumptions}
\crefname{section}{Section}{Sections}
\crefname{equation}{Equation}{Equations}

\title{The Origin of Edge of Stability}

\author{%
  Elon Litman \\
  Stanford University\\
  \texttt{elonlit@stanford.edu}
}

\begin{document}

\maketitle

\begin{abstract}
Full-batch gradient descent on neural networks drives the largest Hessian eigenvalue to the threshold $2/\eta$, where $\eta$ is the learning rate. This phenomenon, the \emph{Edge of Stability}, has resisted a unified explanation: existing accounts establish self-regulation near the edge but do not explain why the trajectory is forced toward $2/\eta$ from arbitrary initialization. We introduce the \emph{edge coupling}, a functional on consecutive iterate pairs whose coefficient is uniquely fixed by the gradient-descent update. Differencing its criticality condition yields a step recurrence with stability boundary $2/\eta$, and a second-order expansion yields a loss-change formula whose telescoping sum forces curvature toward $2/\eta$. The two formulas involve different Hessian averages, but the mean value theorem localizes each to the true Hessian at an interior point of the step segment, yielding exact forcing of the Hessian eigenvalue with no gap. Setting both gradients of the edge coupling to zero classifies fixed points and period-two orbits; near a fixed point, the problem reduces to a function of the half-amplitude alone, which determines which directions support period-two orbits and on which side of the critical learning rate they appear.
\end{abstract}

%% ====================================================================
\section{Introduction}\label{sec:introduction}
%% ====================================================================

\paragraph{The Edge of Stability.}
Classical optimization theory guarantees monotonic loss decrease whenever $\eta < 2/\lambda_{\max}(\nabla^2 L)$, where $\lambda_{\max}(\nabla^2 L)$ is the largest Hessian eigenvalue, called the \emph{sharpness}~\citep{nesterov2004introductory,nocedal2006numerical}. \citet{cohen2021gradient} discovered that in practice the opposite occurs: when a fixed learning rate is used, the sharpness rises during training until it reaches the value $2/\eta$, at which point it saturates and the training loss begins to oscillate on short timescales while continuing to decrease on longer ones. They termed this the \emph{Edge of Stability} (EoS) and documented it across architectures, datasets, and loss functions. In the closely related \emph{catapult phase}~\citep{lewkowycz2020large}, large learning rates initially drive sharpness down before progressive sharpening returns it to $2/\eta$. Progressive sharpening had been studied by \citet{jastrzebski2017three}, \citet{wu2018sgd}, and \citet{ghorbani2019investigation}, who connected it to the Hessian spectrum and implicit selection of flat minima; \citet{lyu2020gradient} established a related bias toward margin maximization. Together, these observations established $2/\eta$ as a universal threshold of full-batch gradient descent. \autoref{fig:eos_motivation} illustrates both phases.

\begin{figure}[t]
\centering
\includegraphics[width=0.48\linewidth]{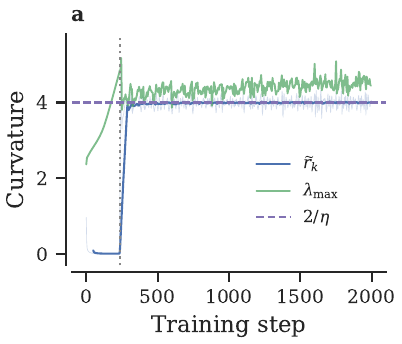}%
\hfill
\includegraphics[width=0.48\linewidth]{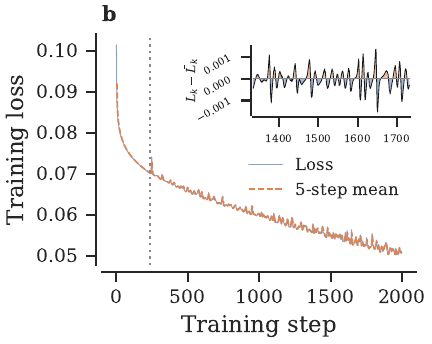}
\caption{\textbf{Edge of Stability on a 3-layer MLP (CIFAR-10).}
Full-batch GD, $\eta=0.5$, GELU activations, MSE loss.
Dotted line marks $t_c$, the first step at which $\widetilde{r}_k\approx 2/\eta$.
\textbf{a},~Effective curvature $\widetilde{r}_k$ (blue) and sharpness $\lambda_{\max}$ (green); both saturate near $2/\eta=4$ (dashed).
\textbf{b},~Training loss (solid) and 5-step running mean (dashed). Inset: detrended loss $L_k-\bar{L}_k$ showing oscillation.}
\label{fig:eos_motivation}
\end{figure}

\paragraph{Prior work.}
Several complementary lines of work have sought to explain why gradient descent saturates at $2/\eta$. One line shows that discrete gradient descent implicitly regularizes toward flat minima through a gradient-norm penalty~\citep{barrett2021implicit,smith2021origin,keskar2017large,hochreiter1997flat}, with related implicit biases in the stochastic, classification, and large-learning-rate settings~\citep{blanc2020implicit,lyu2020gradient,li2019towards}. A second line addresses local dynamics at the edge: \citet{damian2023selfstabilization} proved that cubic Taylor terms create a feedback loop reducing sharpness near $2/\eta$, and \citet{agarwala2022secondorder} and \citet{even2023sgd} obtained analogous results under spectral conditions. Further work has analyzed unstable convergence beyond the classical threshold~\citep{ahn2022understanding,arora2022understanding}, connected curvature to large-scale training instabilities~\citep{gilmer2022loss,lyu2022understanding}, and related convergence rate to sharpness~\citep{ma2022multiscale}.

\paragraph{The edge coupling.}
Despite this progress, existing results~\citep{damian2023selfstabilization,agarwala2022secondorder,even2023sgd} are inherently local: they establish self-regulation near the edge but do not explain why the trajectory is forced toward $2/\eta$ from arbitrary initialization. We take a different approach. Rather than analyzing the dynamics step by step, we ask: is there a scalar functional on consecutive iterate pairs whose criticality conditions encode gradient descent? Consider the family of symmetric couplings
\begin{align}
L(x) + L(y) - \frac{\alpha}{2}\|x-y\|^2.
\end{align}
Setting the $x$-gradient to zero gives $y = x - \alpha^{-1}\nabla L(x)$, which is the gradient-descent update if and only if $\alpha = \eta^{-1}$. We call the resulting functional
\begin{align}
\mathcal{A}_\eta(x,y) \;\triangleq\; L(x) + L(y) - \frac{1}{2\eta}\|x - y\|_2^2
\end{align}
the \emph{edge coupling}: it is attached to one edge of the discrete trajectory, and the threshold it produces governs the Edge of Stability. Viewed through the lens of Hamiltonian mechanics, this functional can be considered a discrete generating function whose criticality conditions encode gradient descent in boundary-value form~\citep{marsden2001discrete,arnol2013mathematical}.

\paragraph{Contributions.}
Every result in this paper follows from the criticality conditions of $\mathcal{A}_\eta$. Setting the $x$-gradient to zero recovers the gradient-descent update; differencing this condition between consecutive steps yields a recurrence with stability boundary $2/\eta$, and a second-order expansion yields a loss-change formula (\cref{sec:euler_action}). These two formulas involve different Hessian averages, but we show that the mean value theorem localizes each to the true Hessian at an interior point of the step segment. Summing the loss-change formula then forces the localized Hessian eigenvalue toward $2/\eta$ exactly (\cref{sec:eos}). Setting both partial gradients to zero classifies all fixed points and period-two orbits; near a fixed point, the problem reduces to a function of the half-amplitude alone, which determines which directions support period-two orbits and on which side of the critical learning rate they appear (\cref{sec:euler_action}). For two-layer linear networks, this construction is width-invariant and the period-doubling branch appears continuously on the large-learning-rate side of the threshold (\cref{prop:linear_Q}). \Cref{sec:edge_mechanisms} explains why the system remains stable at the edge. The appendices extend the forcing theorems to mini-batch SGD (\cref{app:sgd}), to pairs of trajectories (\cref{sec:discrete}), and to continuous time (\cref{sec:continuous}).
%% ====================================================================
\section{The Edge Coupling}\label{sec:euler_action}
%% ====================================================================

We now develop the consequences of the edge coupling $\mathcal{A}_\eta$, referring to consecutive pairs $(w_k, w_{k+1})$ as the \emph{edges} of the trajectory. Since a period-two orbit alternates symmetrically about a center, it is natural to write $(x,y) = (m-a,\, m+a)$, which gives
\begin{align}
\mathcal A_\eta(m-a,m+a)
=
2\Psi_\eta(m,a),
\qquad
\Psi_\eta(m,a)\triangleq\tfrac{1}{2}\bigl(L(m+a)+L(m-a)\bigr)-\tfrac{1}{\eta}\|a\|^2,
\end{align}
This reparametrization separates the center $m$ from the half-amplitude $a$ of the oscillation and will be central to the period-two analysis in \cref{thm:intrinsic_edge_potential}. The first theorem identifies the critical points of $\mathcal{A}_\eta$, and the second derives the dynamical consequences.

\begin{theorem}[Variational characterization of gradient descent]\label{thm:euler_action}
Let $L:\mathbb R^d\to\mathbb R$ be $C^2$, let $\eta>0$, and define
\begin{align}
\mathcal A_\eta(x,y)
\;\triangleq\;
L(x)+L(y)-\frac{1}{2\eta}\|x-y\|_2^2.
\end{align}
Write $\Gamma_\eta \triangleq \{(x,y)\in\mathbb R^d\times\mathbb R^d : \nabla_x\mathcal A_\eta(x,y)=0\}$ for the set of partial critical points in~$x$, and for a gradient-descent trajectory $w_{k+1}=w_k-\eta\nabla L(w_k)$ set $d_k \triangleq w_{k+1}-w_k$.

(i) The edge manifold $\Gamma_\eta$ coincides with the graph of the gradient-descent map:
\begin{align}
\Gamma_\eta
=
\{(x,y): y=x-\eta\nabla L(x)\}.
\end{align}
In particular, every consecutive pair $(w_k,w_{k+1})$ lies on $\Gamma_\eta$.

(ii) The $y$-gradient measures how far the current edge is from closing a period-two orbit. On a GD edge it gives the two-step displacement:
\begin{align}
-\eta\,\nabla_y\mathcal A_\eta(w_k,w_{k+1})
=
w_{k+2}-w_k
=
d_k+d_{k+1}.
\end{align}
Setting both partial gradients to zero therefore imposes the pair of update equations
\begin{align}
\nabla\mathcal A_\eta(x,y)=0
\quad\iff\quad
\begin{cases}
y=x-\eta\nabla L(x),\\
x=y-\eta\nabla L(y),
\end{cases}
\end{align}
so the full critical points of $\mathcal A_\eta$ are the fixed points $(x=y)$
and the period-two orbits $(x\neq y)$ of gradient descent.
\end{theorem}
\noindent\emph{Proof sketch.} The result follows by direct differentiation: $\nabla_x\mathcal{A}_\eta = 0$ rearranges to the GD update, and evaluating $\nabla_y\mathcal{A}_\eta$ on a GD edge yields the two-step displacement. See \cref{app:proofs}.

We now extract the content of \cref{thm:euler_action}. Along the step segment $w_k+\tau d_k$ for $\tau\in[0,1]$, two Hessian averages arise naturally: a uniform average $\bar H_k$ from differencing gradients, and a triangularly weighted average $\widetilde H_k$ from expanding the loss.

\begin{theorem}[Propagator and One-Step Loss Change]\label{thm:propagator_loss}
In the setting of \cref{thm:euler_action}, define the step-averaged Hessians and effective curvature
\begin{align}
\bar H_k \triangleq \int_0^1 \nabla^2L(w_k+\tau d_k)\,d\tau &, \quad
\widetilde H_k \triangleq 2\int_0^1 (1-\tau)\nabla^2L(w_k+\tau d_k)\,d\tau, \\ \quad
&\widetilde r_k \triangleq \frac{d_k^\top \widetilde H_k d_k}{\|d_k\|_2^2}
\;\; (d_k\neq 0).
\end{align}

(i) Differencing the partial criticality condition $\nabla_x\mathcal{A}_\eta = 0$ between consecutive edges yields the step-increment propagator:
\begin{align}
d_{k+1} &= (I-\eta \bar H_k)\,d_k, \\
w_{k+2}-w_k &= (2I-\eta\bar H_k)\,d_k.
\end{align}
Two-step return $w_{k+2}=w_k$ therefore occurs if and only if $\bar H_k d_k = (2/\eta)\,d_k$.

(ii) A second-order expansion of $\mathcal{A}_\eta$ along a partial critical edge gives the one-step loss change:
\begin{align}
L(w_{k+1})-L(w_k)
&=
-\frac{1}{\eta}\|d_k\|_2^2+\frac12\,d_k^\top \widetilde H_k d_k
=
-\frac{\|d_k\|_2^2}{2\eta}\Bigl(2-\eta\widetilde r_k\Bigr).
\end{align}
Summing over $k = 0, \ldots, K{-}1$ telescopes to
\begin{align}
\sum_{k=0}^{K-1}\|d_k\|_2^2\!\left(\frac{2}{\eta}-\widetilde r_k\right)
=
2\bigl(L(w_0)-L(w_K)\bigr).
\end{align}
\end{theorem}
\noindent\emph{Proof sketch.} Part~(i) differences $\nabla_x\mathcal{A}_\eta = 0$ between consecutive edges and applies the fundamental theorem of calculus. Part~(ii) uses the exact Taylor formula with integral remainder and vanishing linear term. See \cref{app:proofs}.

At the Edge of Stability, gradient descent approximately reverses its step at each iteration, producing near-periodic oscillations. We now analyze the period-two orbits that organize this behavior. A period-two orbit near a critical point $\bar w$ has the form $(\bar w - a,\, \bar w + a)$, but as $a$ grows the center shifts away from $\bar w$. The implicit function theorem lets us solve for the true center $m(a)$ and rewrite the problem in terms of the half-amplitude $a$ alone:
\begin{align}
\Phi_\eta(a) = \mathcal{P}(a) - \frac{1}{\eta}\|a\|^2,
\end{align}
where $\mathcal{P}$ depends only on $L$ and not on the learning rate.

\begin{theorem}[Center reduction and the edge eigenproblem]\label{thm:intrinsic_edge_potential}
Let $L$ be $C^4$ near a nondegenerate critical point $\bar w$, and write
$H \triangleq \nabla^2 L(\bar w)$.
There exist neighborhoods $U\ni 0$ and $V\ni \bar w$ and a unique smooth map
$m:U\to V$ such that
\begin{align}\label{eq:center_balance}
m(0)=\bar w,
\qquad
\nabla L(m(a)+a)+\nabla L(m(a)-a)=0
\qquad (a\in U).
\end{align}
The map $m$ is even: $m(-a)=m(a)$.

Define
\begin{align}\label{eq:edge_profile_def}
\mathcal P(a)
\;\triangleq\;
\frac12\Bigl(L(m(a)+a)+L(m(a)-a)\Bigr),
\qquad a\in U.
\end{align}
Then, for every $\eta>0$,
\begin{align}\label{eq:Phi_eta_edge_profile}
\Phi_\eta(a)
\;\triangleq\;
\Psi_\eta(m(a),a)
=
\mathcal P(a)-\frac{1}{\eta}\|a\|^2
\end{align}
is an even function of $a$. The gradient of $\mathcal{P}$ satisfies
\begin{align}\label{eq:edge_profile_gradient}
\nabla \mathcal P(a)
=
\frac12\Bigl(\nabla L(m(a)+a)-\nabla L(m(a)-a)\Bigr),
\end{align}
so the critical points of $\Phi_\eta$ satisfy a nonlinear eigenvalue equation: $a \neq 0$ is a critical point of $\Phi_\eta$ if and only if
\begin{align}\label{eq:nonlinear_eigenproblem}
\nabla \mathcal P(a)=\frac{2}{\eta}\,a.
\end{align}
Full criticality of $\mathcal{A}_\eta$ at $(m(a)-a,\, m(a)+a)$ is equivalent to this same equation. Nearby critical points of $\mathcal A_\eta$ are therefore in bijection with critical points of $\Phi_\eta$: $a=0$ gives the fixed point $(\bar w,\bar w)$, and $a\neq 0$ gives a period-two orbit $x_\eta = m(a)-a$, $y_\eta = m(a)+a$.

The Hessian of $\mathcal{P}$ at the origin is $\nabla^2\mathcal{P}(0) = H$, and therefore
\begin{align}\label{eq:hessian_reduced_profile}
\nabla^2 \Phi_\eta(0)=H-\frac{2}{\eta}I.
\end{align}
This becomes singular when $2/\eta$ enters the spectrum of $H$, which is therefore the spectral threshold for the birth of nontrivial period-two orbits.
\end{theorem}
\noindent\emph{Proof sketch.} The implicit function theorem applied to the center-balance equation yields the center map $m(a)$; the marginal formula and nonlinear eigenvalue equation follow from differentiating the reduced functional. See \cref{app:proofs}.

\noindent\emph{Remark on nondegeneracy.} Deep learning landscapes typically have degenerate minima forming Morse--Bott manifolds (due to overparameterization symmetries). The theorem applies to such settings by restricting to the normal space of the minimum manifold, as we demonstrate for two-layer linear networks in \cref{prop:linear_Q} and \cref{app:linear_Q}.

To use \cref{thm:intrinsic_edge_potential} we need the Taylor expansion of $\mathcal{P}$. Its quartic term $\mathcal Q$ determines whether period-two orbits appear.

\begin{proposition}[Quartic expansion near a fixed point]\label{prop:quartic_jet}
In the setting of \cref{thm:intrinsic_edge_potential}, the center map has expansion
\begin{align}\label{eq:center_map_intrinsic}
m(a)
=
\bar w
-\frac12\,H^{-1}\nabla^3L(\bar w)[a,a,\cdot]
+O(\|a\|^4),
\end{align}
and $\mathcal P$ has expansion
\begin{align}\label{eq:edge_profile_expansion}
\mathcal P(a)
=
L(\bar w)
+\frac12\langle H a,a\rangle
+\frac14\,\mathcal Q(a)
+o(\|a\|^4),
\end{align}
where
\begin{align}\label{eq:Q_fullspace}
\mathcal Q(a)
\;\triangleq\;
\frac16\,\nabla^4L(\bar w)[a,a,a,a]
-\frac12\Bigl\langle
\nabla^3L(\bar w)[a,a,\cdot],\,
H^{-1}\nabla^3L(\bar w)[a,a,\cdot]
\Bigr\rangle .
\end{align}
Consequently,
\begin{align}\label{eq:Phi_eta_expansion_intrinsic}
\Phi_\eta(a)
=
L(\bar w)
+\frac12\Bigl\langle\Bigl(H-\frac{2}{\eta}I\Bigr)a,a\Bigr\rangle
+\frac14\,\mathcal Q(a)
+o(\|a\|^4).
\end{align}
\end{proposition}
\noindent\emph{Proof sketch.} Differentiating the center-balance equation twice at $a = 0$ determines the leading term of $m(a)$, and substituting into the Taylor expansion of $\mathcal{P}$ produces $\mathcal Q$. See \cref{app:proofs}.

With this expansion in hand, the bifurcation problem reduces to finding nontrivial solutions of $\nabla \mathcal{P}(a) = (2/\eta)a$. Near the critical learning rate $\eta_c$, any small solution $a$ must lie close to the kernel $E_c = \ker(H - (2/\eta_c)I)$. The quartic term $\mathcal Q$, restricted to the unit sphere $S(E_c)$, then determines which directions in $E_c$ support bifurcating branches and on which side of $\eta_c$ they appear~\citep{kielhofer2012bifurcation,golubitsky1985singularities}.

\begin{corollary}[Generic branching at the edge]\label{cor:generic_edge_branches}
Fix $\eta_c>0$ and let
\begin{align}
E_c \triangleq \ker\!\Bigl(H-\frac{2}{\eta_c}I\Bigr).
\end{align}
Assume $u\in S(E_c)$ is a nondegenerate critical point of
$\mathcal Q|_{S(E_c)}$ and that $\mathcal Q(u)\neq 0$.
Then there exists a unique local branch of nontrivial period-two
orbits, unique up to swapping the two points, with amplitude
\begin{align}
a(\eta)=\alpha(\eta)\,u+o\!\bigl(\sqrt{|\eta-\eta_c|}\bigr),
\end{align}
where
\begin{align}
\alpha(\eta)^2
=
\frac{\frac{2}{\eta}-\frac{2}{\eta_c}}{\mathcal Q(u)}
+o(|\eta-\eta_c|).
\end{align}
The branch exists on the side
\begin{align}
\Bigl(\frac{2}{\eta}-\frac{2}{\eta_c}\Bigr)\mathcal Q(u)>0.
\end{align}
Moreover, every sufficiently small period-two orbit is
tangent to a critical direction of $\mathcal Q|_{S(E_c)}$.

In particular, when $\dim E_c=1$, this reduces to the scalar
amplitude scaling.
\end{corollary}
\noindent\emph{Proof sketch.} A Lyapunov--Schmidt reduction decomposes the amplitude into $E_c$ and $E_c^\perp$ components; the radial projection of the reduced gradient equation yields the amplitude scaling. See \cref{app:proofs}.

\begin{proposition}[Transverse edge normal form for two-layer linear networks]\label{prop:linear_Q}
For $L_h(W_1,W_2)=\frac12\|W_2W_1-M\|_F^2$ with $W_1\in\mathbb{R}^{h\times d}$, $W_2\in\mathbb{R}^{p\times h}$, and $\operatorname{rank}(M)=r\le h$, the minimum set is a Morse--Bott manifold. The Hessian has a nontrivial kernel corresponding to reparametrization symmetries, so we restrict to the orthogonal complement $\mathcal N = (\ker\nabla^2 L_h(\bar w))^\perp$. The resulting transverse edge theory is \emph{width-invariant}: overparameterization adds only flat directions and leaves the restricted loss unchanged. If $\sigma_1>\sigma_2$, then
\begin{align}
\Phi_{\eta}^{\perp}(t u_c)
\;=\;
L_h(\bar w)+\bigl(\sigma_1-\tfrac1\eta\bigr)t^2-t^4+o(t^4),
\end{align}
so the first period-doubling occurs for $\eta>\eta_c=1/\sigma_1$ and the branch emerges continuously from zero at $\eta_c$ for every $h\ge r$ (\autoref{fig:bifurcation}; proof in \cref{app:linear_Q}).
\end{proposition}

\begin{figure}[t]
\centering
\includegraphics[width=0.48\linewidth]{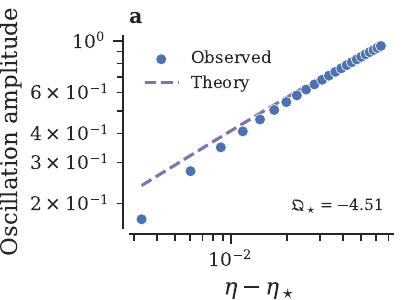}%
\hfill
\includegraphics[width=0.48\linewidth]{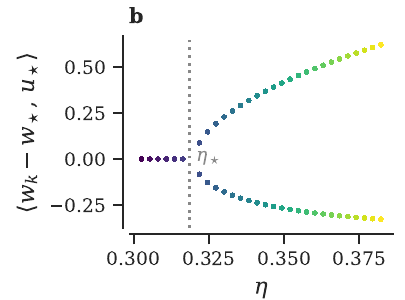}
\caption{\textbf{Continuous onset of period-doubling in a two-layer linear network (\cref{prop:linear_Q}).}
$p\!=\!5$, $h\!=\!3$, $d\!=\!10$, rank-3 target, $n=200$ samples.
\textbf{a},~Period-two amplitude $2\|a(\eta)\|$ vs.\ $\eta - \eta_c$ (log-log).
Observed amplitude (dots) tracks the $\sqrt{\eta - \eta_c}$ scaling predicted by \cref{cor:generic_edge_branches} (dashed);
$\mathcal Q^\perp(u_c) < 0$ shows that the branch appears for $\eta>\eta_c$.
\textbf{b},~Pitchfork diagram: projection $\langle w_k-\bar w,u_c\rangle$ vs.\ $\eta$. Branches emerge continuously at $\eta_c$, confirming that period-doubling begins for $\eta>\eta_c$.}
\label{fig:bifurcation}
\end{figure}

%% ====================================================================
\section{Why {\boldmath$2/\eta$} Appears Everywhere}\label{sec:selfstrain}
%% ====================================================================

\paragraph{Two step-averaged Hessians.}
\Cref{thm:propagator_loss} introduced two Hessian averages along each step segment: a uniform average $\bar H_k$ and a triangularly weighted average $\widetilde H_k$. These arise inevitably from the fundamental theorem of calculus applied to gradients and losses, respectively:
\begin{align}
\nabla L(w_{k+1})-\nabla L(w_k) &= \bar H_k\, d_k, \label{eq:grad_diff_mvt}\\
L(w_{k+1})-L(w_k) &= \nabla L(w_k)^\top d_k + \tfrac{1}{2}\, d_k^\top\widetilde H_k\, d_k. \label{eq:loss_diff_mvt}
\end{align}

\paragraph{The roles of {\boldmath$2/\eta$}.}
Restating \cref{thm:propagator_loss} in terms of these averages reveals $2/\eta$ as a single threshold governing both propagation and descent.

\begin{corollary}[Roles of {\boldmath$2/\eta$}]\label{prop:triple_consequences}
Fix $k$ with $d_k \neq 0$. Step reversal $d_{k+1}=-d_k$ and two-step return $w_{k+2}=w_k$ are both equivalent to
\begin{align}\label{eq:dk_propagator}
\bar H_k d_k=\frac{2}{\eta}\,d_k.
\end{align}
Recall the effective curvature $\widetilde r_k$ from \cref{thm:propagator_loss}. The loss decreases if and only if $\widetilde r_k \le 2/\eta$:\label{eq:step_curvature}
\begin{align}\label{eq:loss_change_exact}
L(w_{k+1})-L(w_k)
= -\frac{\|d_k\|_2^2}{2\eta}\Bigl(2-\eta \widetilde r_k\Bigr).
\end{align}
\end{corollary}

\paragraph{Curvature along the step.}
Two different Hessian averages appear, but they are two views of the same underlying object. To make this precise, we restrict attention to the step direction $u_k \triangleq d_k/\|d_k\|_2$ and define
\begin{align}
q_k(\tau)\triangleq u_k^\top \nabla^2L(w_k+\tau d_k)\,u_k, \qquad \tau \in [0,1].
\end{align}
With $\bar r_k \triangleq d_k^\top \bar H_k d_k/\|d_k\|_2^2$, both $\bar r_k$ and $\widetilde r_k$ are weighted integrals of $q_k$. Since $q_k$ is continuous, each is attained as an exact pointwise value at some interior point of the step segment (\cref{thm:true_sharpness}), and the loss change is well approximated by a proxy that uses only trajectory data:
\begin{align}
L(w_{k+1}) - L(w_k) \;\approx\; -\frac{1}{2\eta}\,d_k^\top(w_{k+2}-w_k).
\end{align}

%% ====================================================================
\section{The Origin of Edge of Stability}\label{sec:eos}
%% ====================================================================

\paragraph{Why the dynamics is forced toward {\boldmath$2/\eta$}.}\label{subsec:forced_to_edge}
The previous section established $2/\eta$ as the threshold governing both step propagation and loss descent. We now show that this threshold is a global attractor: the trajectory is forced to visit it. The argument is a conservation law. The loss-change formula (\cref{prop:triple_consequences}) writes each per-step loss change as the product of $\|d_k\|_2^2$ and the deviation $2/\eta - \widetilde r_k$. Summing over the trajectory, the loss side telescopes to a bounded quantity while the curvature deviations accumulate. Because the total loss drop is finite, the curvature $\widetilde r_k$ cannot stay far from $2/\eta$.

\begin{theorem}[Curvature concentration at {\boldmath$2/\eta$}]\label{thm:edge_balance}
Let $L:\mathbb{R}^d\to\mathbb{R}$ be $C^2$ and bounded below by
$L_{\inf}\triangleq \inf_{w} L(w) > -\infty$.
Run gradient descent $w_{k+1}=w_k-\eta\nabla L(w_k)$ with step size $\eta>0$.
Use the notation $d_k$, $\widetilde H_k$, $\widetilde r_k$ from
\cref{thm:propagator_loss,eq:step_curvature}, and set
$E_K\triangleq \sum_{k=0}^{K-1}\|d_k\|_2^2$.

For every $K\ge 1$, the $\|d_k\|^2$-weighted deviations from $2/\eta$ telescope to the total loss change:
\begin{align}\label{eq:edge_balance_identity}
\sum_{k=0}^{K-1} \|d_k\|_2^2\!\left(\frac{2}{\eta}-\widetilde r_k\right)
= 2\bigl(L(w_0)-L(w_K)\bigr).
\end{align}
Dividing by $E_K > 0$ rewrites this as a weighted average:
\begin{align}\label{eq:edge_balance_weighted}
\frac{\sum_{k=0}^{K-1}\|d_k\|_2^2\,\widetilde r_k}{E_K}
= \frac{2}{\eta} - \frac{2\bigl(L(w_0)-L(w_K)\bigr)}{E_K}.
\end{align}
Since the left-hand side is bounded above by $\max_{k < K} \widetilde r_k$ and below by $\min_{k < K} \widetilde r_k$, we obtain a lower bound on the maximum curvature encountered along the trajectory:
\begin{align}\label{eq:near_edge_forcing}
\max_{0\le k\le K-1}\widetilde r_k
\;\ge\;
\frac{2}{\eta} - \frac{2\bigl(L(w_0)-L_{\inf}\bigr)}{E_K}.
\end{align}
As $E_K$ grows, the right-hand side approaches $2/\eta$, so the trajectory is forced to visit steps with effective curvature arbitrarily close to the threshold. If $L$ is additionally bounded above along the trajectory and $E_K\to\infty$, then the weighted average converges:
\begin{align}\label{eq:edge_concentration}
\frac{\sum_{k=0}^{K-1}\|d_k\|_2^2\,\widetilde r_k}{E_K}\longrightarrow \frac{2}{\eta}.
\end{align}
\end{theorem}
\noindent\emph{Proof sketch.} Telescope the one-step loss-change formula over $K$ steps. See \cref{app:proofs}.

So far we know that $\widetilde r_k$ concentrates near $2/\eta$ in a weighted average sense. The next theorem strengthens this to step-level concentration by decomposing deviations into positive and negative parts.

\begin{theorem}[Concentration near {\boldmath$2/\eta$}]\label{thm:edge_window}
In the setting of \cref{thm:edge_balance}, define
\begin{align}
B_K^- \;\triangleq\; \sum_{k=0}^{K-1}\|d_k\|_2^2\Bigl(\tfrac{2}{\eta}-\widetilde r_k\Bigr)_+,
\qquad
B_K^+ \;\triangleq\; \sum_{k=0}^{K-1}\|d_k\|_2^2\Bigl(\widetilde r_k-\tfrac{2}{\eta}\Bigr)_+.
\end{align}
Then for every $K\ge 1$,
\begin{align}\label{eq:signed_balance}
B_K^- - B_K^+ = 2\bigl(L(w_0)-L(w_K)\bigr).
\end{align}
For any $\delta>0$,
\begin{align}
\sum_{\widetilde r_k\le 2/\eta-\delta}\|d_k\|_2^2 &\;\le\; \frac{2(L(w_0)-L_{\inf})+B_K^+}{\delta}, \label{eq:window_bound_sub}\\
\sum_{\widetilde r_k\ge 2/\eta+\delta}\|d_k\|_2^2 &\;\le\; \frac{B_K^+}{\delta}. \label{eq:window_bound_super}
\end{align}
Hence if $B_\infty^+ \triangleq \lim_{K\to\infty}B_K^+<\infty$, then for every $\delta>0$,
\begin{align}\label{eq:finite_outside_window}
\sum_{k=0}^{\infty}\mathbf{1}_{\{|\widetilde r_k-2/\eta|\ge\delta\}}\|d_k\|_2^2 < \infty,
\end{align}
and if moreover $E_K\to\infty$, the weighted curvature concentrates at $2/\eta$:
\begin{align}\label{eq:window_concentration}
\frac{\sum_{k=0}^{K-1}\mathbf{1}_{\{|\widetilde r_k-2/\eta|<\delta\}}\|d_k\|_2^2}{E_K}
\longrightarrow 1.
\end{align}
No monotone-descent assumption is needed; the controlling quantity is finiteness of $B_\infty^+$.
\end{theorem}
\noindent\emph{Proof sketch.} Decompose $2/\eta - \widetilde{r}_k$ into positive and negative parts; Markov-type estimates yield the window bounds. See \cref{app:proofs}. Panel~(b) of \autoref{fig:edge_balance} validates the loss-change formula step by step: the scatter of actual $\Delta L_k$ against the prediction from $\bar{r}_k$ lies tightly along $y=x$, confirming both \cref{thm:propagator_loss}(ii) and that $\bar r_k$ and $\widetilde r_k$ nearly coincide (\cref{thm:true_sharpness}).

\begin{figure}[t]
\centering
\includegraphics[width=0.48\linewidth]{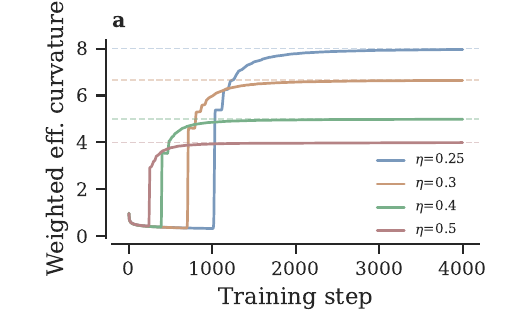}%
\hfill
\includegraphics[width=0.48\linewidth]{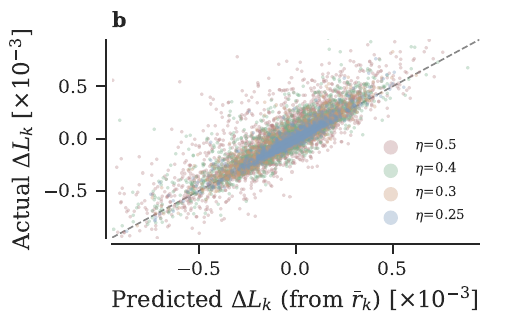}
\caption{\textbf{Validation of \cref{thm:edge_balance,thm:propagator_loss}.}
Four learning rates, $4{,}000$ steps, shared initialization.
\textbf{a},~Weighted average curvature converges to $2/\eta$ (dashed).
\textbf{b},~Actual $\Delta L_k$ vs.\ the proxy $-\frac{1}{2\eta}d_k^\top(w_{k+2}-w_k)$.
Its tightness confirms that $\bar r_k$ and $\widetilde r_k$ nearly coincide on this run.}
\label{fig:edge_balance}
\end{figure}

\paragraph{Exact sharpness forcing on each edge.}
So far the concentration theorems control $\widetilde r_k$ and $\bar r_k$, which are weighted averages of the Hessian along each step segment. A natural question is whether this forcing extends to the true Hessian eigenvalue at an actual point. The answer is affirmative and requires nothing beyond the mean value theorem: along each edge, $\widetilde r_k$ and $\bar r_k$ are realized as exact pointwise values of the directional curvature at specific interior points, so the forcing transfers to the true Hessian eigenvalue at those points with no residual gap.

\begin{theorem}[Localization to the true Hessian]\label{thm:true_sharpness}
Let $u_k = d_k/\|d_k\|_2$ when $d_k \neq 0$. For each such $k$, there exist $\xi_k, \zeta_k \in (0,1)$ such that
\begin{align}\label{eq:edge_localization}
\widetilde r_k = u_k^\top \nabla^2 L(w_k + \xi_k d_k)\, u_k, \qquad \bar r_k = u_k^\top \nabla^2 L(w_k + \zeta_k d_k)\, u_k.
\end{align}
Consequently, $\lambda_{\max}(\nabla^2 L(w_k + \xi_k d_k)) \ge \widetilde r_k$ and $\lambda_{\max}(\nabla^2 L(w_k + \zeta_k d_k)) \ge \bar r_k$. Combined with \cref{thm:edge_balance}:
\begin{align}\label{eq:edge_localized_forcing}
\sum_{k=0}^{K-1} \|d_k\|_2^2 \!\left(\frac{2}{\eta} - \lambda_{\max}\!\bigl(\nabla^2 L(w_k + \xi_k d_k)\bigr)\right) \le 2\bigl(L(w_0) - L(w_K)\bigr),
\end{align}
\begin{align}\label{eq:edge_localized_max}
\max_{0 \le k < K} \lambda_{\max}\!\bigl(\nabla^2 L(w_k + \xi_k d_k)\bigr) \ge \frac{2}{\eta} - \frac{2(L(w_0) - L_{\inf})}{E_K}.
\end{align}
From the propagator (\cref{cor:near_periodic}):
\begin{align}\label{eq:edge_localized_propagator}
\lambda_{\max}\!\bigl(\nabla^2 L(w_k + \zeta_k d_k)\bigr) \ge \frac{2}{\eta} - \frac{\|w_{k+2} - w_k\|_2}{\eta\,\|d_k\|_2}.
\end{align}
\end{theorem}
\noindent\emph{Proof sketch.} Define the scalar restriction $g_k(t) = L(w_k + t\,d_k)$. Taylor's theorem with Lagrange remainder gives $g_k(1) = g_k(0) + g_k'(0) + \frac{1}{2}g_k''(\xi_k)$; comparing with the integral form (\cref{thm:propagator_loss}(ii)) identifies $\widetilde r_k = q_k(\xi_k)$. The mean value theorem applied to $g_k'$ gives $\bar r_k = q_k(\zeta_k)$. The forcing and propagator bounds follow from $q_k \le \lambda_{\max}$ combined with \cref{thm:edge_balance,cor:near_periodic}. See \cref{app:true_sharpness}.

\paragraph{Discrete stability and near-periodicity.}
The recurrence $d_{k+1} = (I - \eta \bar H_k) d_k$ is stable when the eigenvalues of $\bar H_k$ lie in $[0, 2/\eta]$: outside this interval the multiplier $|1 - \eta\lambda|$ exceeds $1$ and the step norm grows. At the Edge of Stability the sharpest eigenvalue sits near $2/\eta$, giving a multiplier near $-1$; the step direction reverses at each iteration and the loss oscillates (formal bound in \cref{thm:discrete_stability}).

In practice, gradient descent does not settle into a period-two orbit but exhibits approximate periodicity~\citep{cohen2021gradient}: the iterate nearly returns after two steps, so $\|w_{k+2} - w_k\| \ll \|w_{k+1} - w_k\|$ (\autoref{fig:near_periodicity} in \cref{app:proofs}). This condition alone, without any Hessian computation, ensures that the directional curvature of $\bar H_k$ along the step is close to $2/\eta$.

\begin{corollary}[Near-Periodicity Implies Near-Critical Sharpness]\label{cor:near_periodic}
The directional curvature of $\bar{H}_k$ along the step direction satisfies
\begin{align}\label{eq:near_periodic_bound}
\left|\frac{d_k^\top \bar{H}_k\, d_k}{\|d_k\|_2^2} - \frac{2}{\eta}\right| \le \frac{\|w_{k+2} - w_k\|_2}{\eta\,\|d_k\|_2}.
\end{align}
\end{corollary}
\noindent\emph{Proof sketch.} Express the two-step displacement in terms of gradients; Cauchy--Schwarz completes the bound. See \cref{app:proofs}.

%% ====================================================================
\section{Stability Mechanisms at the Edge}\label{sec:edge_mechanisms}
%% ====================================================================

\cref{sec:eos} explains where $2/\eta$ comes from. Two questions remain: why does the system not diverge at the edge, and why does it stay there? The answer splits into two regimes (proofs in \cref{app:recoil_oscillatory}).

\paragraph{Growth above the threshold.}
Write $\bar r_k \triangleq d_k^\top \bar H_k d_k / \|d_k\|_2^2$.

\begin{proposition}[Growth above the threshold]\label{prop:supercritical_recoil}
$\langle d_{k+1}, d_k\rangle = (1-\eta\bar r_k)\|d_k\|_2^2$. Hence if $\bar r_k \ge 2/\eta + \delta$, then $\|d_{k+1}\|_2 \ge (1+\eta\delta)\|d_k\|_2$, and a sustained above-threshold run of length $t-s$ yields geometric growth: $\|d_t\|_2 \ge (1+\eta\delta)^{t-s}\|d_s\|_2$.
\end{proposition}

On any bounded trajectory, step-norm growth is impossible, so the dynamics cannot remain above the threshold.

\paragraph{Oscillatory cancellation.}
The recoil explains why the system cannot stay above $2/\eta$. Inside the stability window, the question is whether the near-critical multiplier $m_k \approx -1$ causes secular drift.

\begin{theorem}[Oscillatory cancellation]\label{thm:robust_oscillatory}
Consider $x_{k+1} = m_k x_k - \eta u_k$, $x_0 = 0$, with $m_k \in [-1, 0]$ for all $k$. Then for every $T \ge 1$,
\begin{align}\label{eq:robust_oscillatory_bound}
|x_T| \le \eta\left(|u_{T-1}| + \sum_{k=0}^{T-2}|u_{k+1}-u_k|\right).
\end{align}
\end{theorem}

The bound holds for arbitrary time-varying multipliers across the entire oscillatory stability window. When the curvature rises above $2/\eta$, \cref{prop:supercritical_recoil} forces the dynamics back through expanding sign reversals; inside the window, \cref{thm:robust_oscillatory} prevents secular growth. Together these explain the empirical observation~\citep{cohen2021gradient, damian2023selfstabilization} that sharpness hovers near $2/\eta$: excursions above the threshold drive parameters into regions of lower curvature, while near-threshold oscillations cancel rather than drift.

%% ====================================================================
\section{Conclusion}\label{sec:conclusion}
%% ====================================================================

The edge coupling $\mathcal{A}_\eta$ provides a unified account of the Edge of Stability. Its criticality conditions yield the step recurrence and the loss-change formula; summing the latter forces curvature to $2/\eta$ globally (\cref{thm:edge_balance}), resolving the locality limitation of prior work~\citep{damian2023selfstabilization}. The mean value theorem then localizes each average to the true Hessian eigenvalue at an interior point of each step segment, yielding exact forcing with no residual gap. The center reduction in \cref{thm:intrinsic_edge_potential} reduces period-two bifurcation to a nonlinear eigenproblem; for two-layer linear networks this reduction is width-invariant and the period-doubling branch appears continuously on the large-learning-rate side of the threshold (\cref{prop:linear_Q}).

\paragraph{Limitations and future directions.}
The global theorems require $C^2$ smoothness and the bifurcation analysis requires $C^4$, which excludes losses with ReLU activations unless smoothed. The forcing theorems guarantee that curvature visits $2/\eta$ but do not predict how long the system spends there or how the loss decreases in the EoS phase. Connecting the two-trajectory stability bound (\cref{sec:discrete}) to curvature-aware generalization bounds in the spirit of \citet{hardt2016train} is a natural next step.

\bibliographystyle{plainnat}
\bibliography{references}

\newpage
\appendix

%% ====================================================================
\section{Proofs from the Main Text}\label{app:proofs}
%% ====================================================================

This appendix collects the proofs of all results stated in \cref{sec:euler_action,sec:selfstrain,sec:eos,sec:edge_mechanisms}. We also develop two results referenced by the main text: a spectral analysis of the edge coupling at a fixed point (\cref{prop:edge_threshold}), which shows that the Hessian of $\mathcal{A}_\eta$ at a diagonal critical point $(\bar w, \bar w)$ becomes indefinite when an eigenvalue of $\nabla^2 L(\bar w)$ exceeds $2/\eta$; and a discrete stability bound (\cref{thm:discrete_stability}), which controls how far the step norm can grow when the curvature temporarily leaves the interval $[0, 2/\eta]$.

\paragraph{Proof of \cref{thm:euler_action}.}
\begin{proof}
\emph{(i)} Since $\nabla_x\mathcal A_\eta(x,y)=\nabla L(x)-\frac1\eta(x-y)$, the condition
$\nabla_x\mathcal A_\eta(x,y)=0$ is equivalent to $y=x-\eta\nabla L(x)$.

\emph{(ii)} We have $\nabla_y\mathcal A_\eta(x,y)=\nabla L(y)+\frac1\eta(x-y)$.
Evaluating at $(x,y)=(w_k,w_{k+1})$ and using $w_{k+1}-w_k=-\eta\nabla L(w_k)$ gives
\begin{align}
-\eta\nabla_y\mathcal A_\eta(w_k,w_{k+1})
&=
-\eta\nabla L(w_{k+1})-(w_k-w_{k+1})\\
&=
(w_{k+2}-w_{k+1})+(w_{k+1}-w_k)\\
&=
w_{k+2}-w_k.
\end{align}
The critical-point characterization follows by imposing both
$\nabla_x\mathcal A_\eta=0$ and $\nabla_y\mathcal A_\eta=0$.
\end{proof}

\paragraph{Proof of \cref{thm:propagator_loss} (Propagator and One-Step Loss Change).}
The strategy for part~(i) is to difference the criticality condition $\nabla_x \mathcal{A}_\eta = 0$ between two consecutive edges, which converts a gradient identity into a step-increment recurrence. For part~(ii), we slide one endpoint of the edge coupling from $w_k$ to $w_{k+1}$ along the step segment; the vanishing of the first derivative at $\tau = 0$ (which is the criticality condition itself) turns the Taylor expansion into a pure second-order expression.

\begin{proof}
\emph{(i)} Both consecutive pairs $(w_k,w_{k+1})$ and $(w_{k+1},w_{k+2})$ lie on $\Gamma_\eta$, so differencing the partial criticality condition between them and substituting $\nabla L(w_{k+1})-\nabla L(w_k)=\bar H_k d_k$ gives
\begin{align}
0
&=
\nabla_x\mathcal A_\eta(w_{k+1},w_{k+2})
-\nabla_x\mathcal A_\eta(w_k,w_{k+1})\\
&=
\bar H_k d_k
-\frac1\eta(d_k-d_{k+1}),
\end{align}
which rearranges to $d_{k+1}=(I-\eta\bar H_k)d_k$. The two-step displacement follows by adding $d_k$ to both sides.

\emph{(ii)} Fix $k$ and consider sliding the first argument of $\mathcal{A}_\eta$ along the step: define $f_k(\tau)\triangleq \mathcal A_\eta(w_k+\tau d_k,w_{k+1})$ for $\tau\in[0,1]$.
At $\tau = 0$ we are at the actual edge $(w_k, w_{k+1})$, and at $\tau = 1$ both arguments coincide at $w_{k+1}$. The partial criticality condition ensures that the linear term vanishes: $f_k'(0) = 0$. The boundary values and second derivative are
\begin{align}
f_k(0)
&=
L(w_k)+L(w_{k+1})-\frac{1}{2\eta}\|d_k\|_2^2,\\
f_k(1)
&=
\mathcal A_\eta(w_{k+1},w_{k+1})
=
2L(w_{k+1}),\\
f_k''(\tau)&=d_k^\top \nabla^2L(w_k+\tau d_k)d_k-\frac1\eta\|d_k\|_2^2.
\end{align}
The Taylor formula with integral remainder, $f_k(1)-f_k(0)=\int_0^1 (1-\tau)f_k''(\tau)\,d\tau$, then gives after substitution
\begin{align}
L(w_{k+1})-L(w_k)
=
-\frac{1}{\eta}\|d_k\|_2^2+\frac12\,d_k^\top \widetilde H_k d_k.
\end{align}
The telescoping identity follows by summing over $k$.
\end{proof}

The forcing and concentration theorems in the main text describe the trajectory globally, but the edge coupling also contains local information at a fixed point. The Hessian of $\mathcal{A}_\eta$ at a diagonal critical point $(\bar w, \bar w)$ acts on the product space $\mathbb{R}^d \times \mathbb{R}^d$, and its eigendirections split into two natural families: diagonal directions $(u, u)$ that move both iterates together, and anti-diagonal directions $(u, -u)$ that push them apart. The following proposition shows that the diagonal directions are always stable (they see only the curvature $H$), while the anti-diagonal directions become unstable when any eigenvalue of $H$ crosses $2/\eta$. This provides a spectral interpretation of the Edge of Stability directly from the second-order structure of the edge coupling.

\begin{proposition}[Spectral threshold of the edge coupling]\label{prop:edge_threshold}
Let $\bar w$ be a critical point of $L$ with $H=\nabla^2L(\bar w)$.
Then $(\bar w,\bar w)$ is a critical point of $\mathcal A_\eta$, with Hessian
\begin{align}
\nabla^2\mathcal A_\eta(\bar w,\bar w)
=
\begin{pmatrix}
H-\frac1\eta I & \frac1\eta I\\
\frac1\eta I & H-\frac1\eta I
\end{pmatrix}.
\end{align}
Along the diagonal and anti-diagonal directions, this reduces to
\begin{align}
\nabla^2\mathcal A_\eta(\bar w,\bar w)[(u,u),(u,u)]
&=
2\,u^\top H u,\\
\nabla^2\mathcal A_\eta(\bar w,\bar w)[(u,-u),(u,-u)]
&=
2\,u^\top\!\left(H-\frac{2}{\eta}I\right)u.
\end{align}
The first expression is always nonnegative at a local minimum, so diagonal perturbations are stable. The second changes sign when $\lambda_{\max}(H)$ exceeds $2/\eta$: at that threshold, the fixed point of the edge coupling loses stability to perturbations that split the two iterates apart, which is the onset of period-two oscillation.
\end{proposition}
\begin{proof}
Differentiating $\mathcal A_\eta$ twice gives
\begin{align}
\nabla^2\mathcal A_\eta(x,y)
=
\begin{pmatrix}
\nabla^2L(x)-\frac1\eta I & \frac1\eta I\\
\frac1\eta I & \nabla^2L(y)-\frac1\eta I
\end{pmatrix}.
\end{align}
Evaluating at $(\bar w,\bar w)$ gives the stated formula, and the diagonal and
anti-diagonal restrictions follow by direct substitution.
\end{proof}

\paragraph{Proof of \cref{thm:intrinsic_edge_potential}.}
\begin{proof}
Define the center-balance map and the symmetrized loss by
\begin{align}
F(m,a)&\triangleq \frac12\bigl(\nabla L(m+a)+\nabla L(m-a)\bigr),\\
G(m,a)&\triangleq \frac12\bigl(L(m+a)+L(m-a)\bigr).
\end{align}
Since $F(\bar w,0)=0$ and $D_mF(\bar w,0)=H$ is invertible, the implicit function theorem gives a unique smooth map
$m(a)$ near $a=0$ with $m(0)=\bar w$ and $F(m(a),a)=0$, proving
\eqref{eq:center_balance}. Because $F(m,-a)=F(m,a)$, uniqueness implies
$m(-a)=m(a)$.

Setting $\mathcal P(a)\triangleq G(m(a),a)$ gives
\begin{align}
\Phi_\eta(a)
=
\Psi_\eta(m(a),a)
=
\mathcal P(a)-\frac{1}{\eta}\|a\|^2,
\end{align}
which is \eqref{eq:Phi_eta_edge_profile}. Since $D_mG(m(a),a)=F(m(a),a)=0$, the $m$-dependence drops out when differentiating $\mathcal P$, and the chain rule gives
\begin{align}
D\mathcal P(a)[h]
&=
D_aG(m(a),a)[h]
=
\frac12\Bigl\langle \nabla L(m(a)+a)-\nabla L(m(a)-a),\,h\Bigr\rangle,
\end{align}
which is \eqref{eq:edge_profile_gradient}. It follows that $\nabla\Phi_\eta(a)=\nabla\mathcal P(a)-\frac{2}{\eta}a$, proving \eqref{eq:nonlinear_eigenproblem}. Since $F(m(a),a)=0$ already
imposes the $m$-criticality equation, $\nabla\Phi_\eta(a)=0$ is equivalent
to full criticality of $\Psi_\eta$, hence of $\mathcal A_\eta$ under the
change of variables $(x,y)=(m-a,m+a)$.

Differentiating \eqref{eq:edge_profile_gradient} at $a=0$ and using $Dm(0)=0$ (since $m$ is even) gives $\nabla^2\mathcal P(0)=H$, and therefore
$\nabla^2\Phi_\eta(0)=H-\frac{2}{\eta}I$.
\end{proof}

\paragraph{Proof of \cref{prop:quartic_jet}.}
\begin{proof}
Because $m$ is even, $Dm(0)=0$. Differentiating the identity
$F(m(a),a)=0$ twice at $a=0$ and solving for the second derivative of the center map gives
\begin{align}
H\,D^2m(0)[h,h]+\nabla^3L(\bar w)[h,h,\cdot]&=0,\\
D^2m(0)[h,h]
&=
-H^{-1}\nabla^3L(\bar w)[h,h,\cdot].
\end{align}
Taylor expanding the even map $m$ then yields \eqref{eq:center_map_intrinsic}:
\begin{align}
m(a)
=
\bar w-\frac12H^{-1}\nabla^3L(\bar w)[a,a,\cdot]+O(\|a\|^4).
\end{align}

To obtain the quartic jet, set $p(a)\triangleq m(a)-\bar w = O(\|a\|^2)$.
Taylor expanding $L(\bar w+p\pm a)$ through quartic order and averaging gives
\begin{align}
\mathcal P(a)
&=
L(\bar w)
+\frac12\langle H a,a\rangle
+\frac12\langle H p,p\rangle \\
&+\frac12\Bigl\langle \nabla^3L(\bar w)[a,a,\cdot],\,p\Bigr\rangle
+\frac1{24}\nabla^4L(\bar w)[a,a,a,a]
+o(\|a\|^4).
\end{align}
Substituting $p(a) = -\frac12H^{-1}\nabla^3L(\bar w)[a,a,\cdot] + o(\|a\|^2)$ and collecting quartic terms yields $\frac14\,\mathcal Q(a)$, proving
\eqref{eq:edge_profile_expansion} and hence \eqref{eq:Phi_eta_expansion_intrinsic}.
\end{proof}

\paragraph{Proof of \cref{cor:generic_edge_branches} (Generic branching at the edge).}
\begin{proof}
Write $\mu \triangleq \frac{2}{\eta_c}-\frac{2}{\eta}$ for the bifurcation parameter. By \cref{thm:intrinsic_edge_potential}, the reduced functional has the expansion
\begin{align}
\Phi_\eta(a)
=
L(\bar w)
+\frac12\langle(H-\tfrac{2}{\eta}I)a,a\rangle
+\frac14\mathcal Q(a)
+o(\|a\|^4+|\mu|\,\|a\|^2).
\end{align}

We perform a Lyapunov--Schmidt reduction. Decompose $a=\xi+b$ with $\xi\in E_c$ and $b\in E_c^\perp$.
Since $H-\frac{2}{\eta_c}I$ is invertible on $E_c^\perp$,
the implicit function theorem solves the $E_c^\perp$-component of
$\nabla\Phi_\eta(\xi+b)=0$ uniquely as $b=b(\xi,\eta)=O(\|\xi\|^3+|\mu|\,\|\xi\|)$. Substituting back reduces the problem to a potential on $E_c$ alone:
\begin{align}
\varphi_\eta(\xi)
=
L(\bar w)
+\frac12\mu\|\xi\|^2
+\frac14\mathcal Q(\xi)
+o(\|\xi\|^4+|\mu|\,\|\xi\|^2).
\end{align}

Writing $\xi=\alpha v$ with $\alpha\ge 0$ and $v\in S(E_c)$ separates the amplitude from the direction. The gradient of the reduced potential is
\begin{align}
\nabla\varphi_\eta(\alpha v)
=
\mu\alpha v+\frac14\alpha^3\nabla\mathcal Q(v)
+o(\alpha^3+|\mu|\alpha).
\end{align}
The tangential projection to $S(E_c)$ forces $\nabla_{S}(\mathcal Q|_{S(E_c)})(v)=o(1)$,
so every small nontrivial critical point has its direction tending to a critical point of
$\mathcal Q|_{S(E_c)}$. If $u$ is nondegenerate, the implicit
function theorem gives a unique nearby direction branch $v(\eta)\to u$.

For the amplitude, we project radially and use Euler's identity $\langle v,\nabla\mathcal Q(v)\rangle = 4\mathcal Q(v)$ for the homogeneous quartic to obtain
\begin{align}
0
&=
\mu+\alpha^2\mathcal Q(u)+o(|\mu|),\\
\alpha(\eta)^2
&=
\frac{\frac{2}{\eta}-\frac{2}{\eta_c}}{\mathcal Q(u)}
+o(|\eta-\eta_c|).
\end{align}
This is real and positive on the side where $\bigl(\frac{2}{\eta}-\frac{2}{\eta_c}\bigr)\mathcal Q(u)>0$.
\end{proof}

\paragraph{Proof of \cref{prop:triple_consequences} (Roles of {\boldmath$2/\eta$}).}
\begin{proof}
\Cref{thm:propagator_loss} gives the propagator \eqref{eq:dk_propagator} and the loss-change formula \eqref{eq:loss_change_exact}. For the two-step return, we compute the displacement directly:
\begin{align}
w_{k+2}-w_k = d_{k+1}+d_k = (I-\eta\bar H_k)d_k+d_k = (2I-\eta\bar H_k)d_k.
\end{align}
Since $d_k \neq 0$, we have $w_{k+2}=w_k$ if and only if $\bar H_k d_k=(2/\eta)d_k$. The descent/ascent boundary follows from $\|d_k\|_2^2/(2\eta)>0$.
\end{proof}

\paragraph{Proof of \cref{thm:edge_balance} (Curvature concentration at {\boldmath$2/\eta$}).}
\begin{proof}
The loss-change formula (\cref{prop:triple_consequences}) writes each step as $L(w_k)-L(w_{k+1}) = \frac{\|d_k\|_2^2}{2}(\frac{2}{\eta}-\widetilde r_k)$. Summing from $k=0$ to $K-1$, the left side telescopes while the right side accumulates:
\begin{align}
L(w_0)-L(w_K) &= \frac{1}{2}\sum_{k=0}^{K-1}\|d_k\|_2^2\!\left(\frac{2}{\eta}-\widetilde r_k\right),
\end{align}
which is \cref{eq:edge_balance_identity}. Dividing by $E_K>0$ yields
\cref{eq:edge_balance_weighted}. The left side of \cref{eq:edge_balance_weighted} is a weighted average of
$\widetilde r_k$, so it is bounded above by $\max_{k<K}\widetilde r_k$. Using $L(w_K)\ge L_{\inf}$ then gives the forcing bound and the asymptotic concentration:
\begin{align}
\max_{k<K}\widetilde r_k &\;\ge\;
\frac{2}{\eta} - \frac{2(L(w_0)-L_{\inf})}{E_K},\\
\left|\frac{\sum_{k<K}\|d_k\|_2^2\,\widetilde r_k}{E_K} - \frac{2}{\eta}\right|
&= \frac{2|L(w_0)-L(w_K)|}{E_K} \\
&\;\le\; \frac{2\max\{|L(w_0)-L_{\inf}|,\,|L_{\sup}-L(w_0)|\}}{E_K} \to 0.
\end{align}
\end{proof}

\paragraph{Proof of \cref{thm:edge_window} (Concentration within a window of {\boldmath$2/\eta$}).}
\begin{proof}
Write $x_k \triangleq 2/\eta - \widetilde r_k$ and decompose $x = x_+ - (-x)_+$. Combined with the telescoping identity (\cref{thm:edge_balance}), this gives \cref{eq:signed_balance}:
\begin{align}
B_K^- - B_K^+ = \sum_{k=0}^{K-1}\|d_k\|_2^2 x_k = 2\bigl(L(w_0)-L(w_K)\bigr).
\end{align}

For the subcritical bound \cref{eq:window_bound_sub}, when $\widetilde r_k\le 2/\eta-\delta$ we have $(2/\eta-\widetilde r_k)_+\ge\delta$; for the supercritical bound \cref{eq:window_bound_super}, when $\widetilde r_k\ge 2/\eta+\delta$ we have $(\widetilde r_k-2/\eta)_+\ge\delta$. Markov-type estimates then give
\begin{align}
\delta\sum_{\widetilde r_k\le 2/\eta-\delta}\|d_k\|_2^2
&\;\le\; B_K^-
= 2\bigl(L(w_0)-L(w_K)\bigr)+B_K^+
\;\le\; 2\bigl(L(w_0)-L_{\inf}\bigr)+B_K^+,\\
\delta\sum_{\widetilde r_k\ge 2/\eta+\delta}\|d_k\|_2^2 &\;\le\; B_K^+.
\end{align}
Summing both inequalities and passing to $K\to\infty$ under $B_\infty^+<\infty$ gives
\cref{eq:finite_outside_window}. Dividing by $E_K\to\infty$ yields \cref{eq:window_concentration}.
\end{proof}

The step recurrence $d_{k+1} = (I - \eta \bar H_k) d_k$ from \cref{thm:propagator_loss} has the form of a discrete propagator with time-varying coefficient. When all eigenvalues of $\bar H_k$ lie in the interval $[0, 2/\eta]$, each one-step multiplier $|1 - \eta \lambda|$ is at most one, so the step norm cannot grow. Outside this interval, however, some multiplier exceeds one and the step norm inflates. The question is: by how much? The following definition measures the degree of instability at each step, and the theorem shows that the cumulative effect of transient excursions outside $[0, 2/\eta]$ is controlled by their sum.

\begin{definition}[Excursion beyond the stability window]\label{def:supercritical}
For the recurrence $d_{k+1} = (I - \eta A_k) d_k$, define
\begin{align}\label{eq:supercritical}
\kappa_k \triangleq \max\bigl\{0,\; \eta\lambda_{\max}(A_k) - 2,\; -\eta\lambda_{\min}(A_k)\bigr\}.
\end{align}
When $\kappa_k = 0$, the spectrum lies entirely within $[0, 2/\eta]$ and the one-step map is a contraction.
\end{definition}

\begin{theorem}[Discrete Stability Bound]\label{thm:discrete_stability}
Assume $A_k$ is symmetric for all $k$. Then the discrete propagator satisfies
\begin{align}\label{eq:discrete_propagator_bound}
\|\mathcal{T}[k, s]\|_{\mathrm{op}} \le \exp\!\left(\sum_{r=s}^{k-1} \kappa_r\right),
\end{align}
and consequently the discrete strain satisfies
\begin{align}\label{eq:discrete_strain_bound}
\|\delta_k\|_2 \le \eta \sum_{s=0}^{k-1} \exp\!\left(\sum_{r=s+1}^{k-1} \kappa_r\right) \|f_s\|_2.
\end{align}
\end{theorem}
\begin{proof}
For a symmetric matrix $A$ with eigenvalues $\{\lambda_i\}$, the operator norm of the one-step propagator is $\|I - \eta A\|_{\mathrm{op}} = \max_i |1 - \eta\lambda_i|$. When $\eta\lambda_i \in [0,2]$ the multiplier satisfies $|1-\eta\lambda_i| \le 1$; when $\eta\lambda_i > 2$ we have $|1-\eta\lambda_i| = \eta\lambda_i - 1 \le 1 + \kappa$; and when $\lambda_i < 0$ we have $|1-\eta\lambda_i| = 1 + \eta|\lambda_i| \le 1 + \kappa$. In every case $\|I - \eta A_k\|_{\mathrm{op}} \le 1 + \kappa_k$. Applying this bound stepwise:
\begin{align}
\|\mathcal{T}[k, s]\|_{\mathrm{op}} \le \prod_{r=s}^{k-1} (1 + \kappa_r) \le \exp\!\left(\sum_{r=s}^{k-1} \kappa_r\right),
\end{align}
using $\log(1+x) \le x$. The strain bound follows from iterating the recurrence
\begin{align}
\delta_{k+1} = (I - \eta A_k)\delta_k - \eta f_k
\end{align}
to obtain the variation-of-constants representation
\begin{align}
\delta_k = -\eta \sum_{s=0}^{k-1} \mathcal{T}[k, s{+}1]\, f_s,
\end{align}
then applying the propagator bound \cref{eq:discrete_propagator_bound} together with the triangle inequality.
\end{proof}

\paragraph{Proof of \cref{cor:near_periodic} (Near-Periodicity Implies Near-Critical Sharpness).}
\begin{proof}
By the mean value theorem, $\bar{H}_k d_k = \nabla L(w_{k+1}) - \nabla L(w_k)$. We rewrite the gradient difference in terms of the gradient sum to extract the $2/\eta$ threshold:
\begin{align}
\bar{H}_k d_k &= \nabla L(w_{k+1}) - \nabla L(w_k) \\
&= \bigl(\nabla L(w_{k+1}) + \nabla L(w_k)\bigr) - 2\nabla L(w_k) \\
&= -\frac{1}{\eta}(w_{k+2} - w_k) + \frac{2}{\eta}\,d_k.
\end{align}
Taking the inner product with $d_k$, dividing by $\|d_k\|_2^2$, and applying Cauchy--Schwarz to the displacement term gives \cref{eq:near_periodic_bound}.
\end{proof}

\begin{figure}[htbp]
\centering
\includegraphics[width=0.48\linewidth]{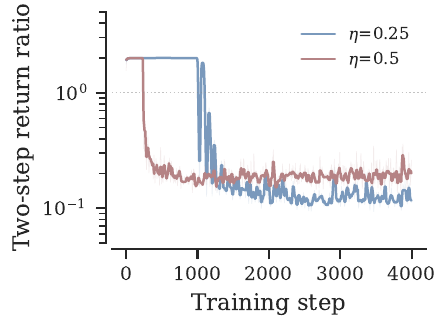}
\caption{\textbf{Two-step return ratio.}
$\|w_{k+2}-w_k\|/\|d_k\|$ vs.\ training step for two learning rates (solid: rolling median; faint: raw). Before EoS onset the ratio exceeds $1$, reflecting the progressive sharpening phase in which consecutive steps reinforce rather than reverse. At EoS onset, the ratio drops from $O(1)$ toward $\sim 0.15$, indicating
approximate (but not exact) period-two behavior; by
\cref{cor:near_periodic}, this directly bounds how close the directional
curvature of $\bar{H}_k$ is to $2/\eta$.}
\label{fig:near_periodicity}
\end{figure}

%% ====================================================================
\section{From Hessian Averages to True Sharpness}\label{app:true_sharpness}
%% ====================================================================

The concentration theorems of \cref{sec:eos} show that the step-averaged curvatures $\widetilde r_k$ and $\bar r_k$ are forced toward $2/\eta$, but the Edge of Stability as observed experimentally involves the largest Hessian eigenvalue at actual points in parameter space. This appendix shows that each averaged curvature is in fact the exact directional curvature of the true Hessian at a specific interior point of the step segment, so the forcing transfers with no residual error.

\begin{definition}[Curvature along a step]\label{def:curvature_profile}
Fix a step $k$ with $d_k \neq 0$ and set $u_k \triangleq d_k / \|d_k\|_2$. Define
\begin{align}
q_k(\tau) \triangleq u_k^\top \nabla^2 L(w_k + \tau d_k)\, u_k, \qquad \tau \in [0,1].
\end{align}
\end{definition}

The uniform average of $q_k$ gives $\bar r_k$ and the triangularly weighted average gives $\widetilde r_k$. Since these are averages of a continuous function over $[0,1]$, the mean value theorem guarantees that each is attained as a pointwise value of $q_k$.

\begin{proof}[Proof of \cref{thm:true_sharpness}]
Define the scalar edge restriction $g_k(t) \triangleq L(w_k + t\,d_k)$ for $t \in [0,1]$. Its derivatives are
\begin{align}
g_k'(t) = \langle \nabla L(w_k + t\,d_k),\, d_k\rangle, \qquad
g_k''(t) = d_k^\top \nabla^2 L(w_k + t\,d_k)\, d_k = \|d_k\|_2^2\, q_k(t).
\end{align}
The gradient-descent update gives $g_k'(0) = \langle \nabla L(w_k), d_k\rangle = -\|d_k\|_2^2/\eta$.

\emph{Localization of $\widetilde r_k$.} Taylor's theorem with Lagrange remainder gives
\begin{align}
g_k(1) = g_k(0) + g_k'(0) + \tfrac{1}{2}\,g_k''(\xi_k)
\end{align}
for some $\xi_k \in (0,1)$. The integral form of the same expansion (\cref{thm:propagator_loss}(ii)) gives
\begin{align}
g_k(1) - g_k(0) = g_k'(0) + \tfrac{1}{2}\|d_k\|_2^2\,\widetilde r_k.
\end{align}
Comparing the two and dividing by $\frac{1}{2}\|d_k\|_2^2$ yields $\widetilde r_k = q_k(\xi_k) = u_k^\top \nabla^2 L(w_k + \xi_k d_k)\, u_k$.

\emph{Localization of $\bar r_k$.} The ordinary mean value theorem applied to $h_k(t) \triangleq g_k'(t)$ gives $h_k(1) - h_k(0) = h_k'(\zeta_k)$
for some $\zeta_k \in (0,1)$. The left side equals $\langle \nabla L(w_{k+1}) - \nabla L(w_k),\, d_k\rangle = \|d_k\|_2^2\,\bar r_k$; the right side equals $g_k''(\zeta_k) = \|d_k\|_2^2\,q_k(\zeta_k)$. Dividing gives $\bar r_k = q_k(\zeta_k) = u_k^\top \nabla^2 L(w_k + \zeta_k d_k)\, u_k$.

\emph{Eigenvalue bounds.} Since $u_k$ is a unit vector, $u_k^\top \nabla^2 L(z)\, u_k \le \lambda_{\max}(\nabla^2 L(z))$ for any $z$.

\emph{Forcing inequality.} Substituting $\widetilde r_k \le \lambda_{\max}(\nabla^2 L(w_k + \xi_k d_k))$ into \cref{eq:edge_balance_identity} gives \cref{eq:edge_localized_forcing}. Using $L(w_K) \ge L_{\inf}$ gives \cref{eq:edge_localized_max}.

\emph{Propagator bound.} From \cref{cor:near_periodic}, $\bar r_k \ge 2/\eta - \|w_{k+2}-w_k\|_2/(\eta\|d_k\|_2)$. Since $\bar r_k \le \lambda_{\max}(\nabla^2 L(w_k + \zeta_k d_k))$, the bound \cref{eq:edge_localized_propagator} follows.
\end{proof}

%% ====================================================================
\section{Growth above the threshold and oscillatory cancellation}\label{app:recoil_oscillatory}%% ====================================================================

This appendix proves the two stability mechanisms from \cref{sec:edge_mechanisms}, which together explain why the dynamics remains bounded at the edge.

The first mechanism is growth above the threshold (\cref{prop:supercritical_recoil}). When the curvature exceeds $2/\eta$, the one-step multiplier along the step direction has magnitude greater than one, so the step norm grows geometrically. On any bounded trajectory this growth is unsustainable, which means the dynamics cannot remain above the threshold for more than a few steps before being forced back below it.

The second mechanism is oscillatory cancellation (\cref{thm:robust_oscillatory}). Once the curvature is back inside the stability window $[0, 2/\eta]$, the multiplier is close to $-1$, so the step direction alternates sign at each iteration. One might worry that small asymmetries in this alternation could accumulate into secular drift. The cancellation theorem shows that this does not happen: the total displacement is controlled by the last forcing value plus the total variation of the forcing sequence, regardless of how the multipliers vary within the window.

\begin{proof}[Proof of \cref{prop:supercritical_recoil}]
The propagator identity gives $d_{k+1}=(I-\eta\bar H_k)d_k$. Taking the inner product with $d_k$:
\begin{align}
\langle d_{k+1},d_k\rangle
= \|d_k\|_2^2 - \eta\,d_k^\top\bar H_k d_k
= (1-\eta\bar r_k)\|d_k\|_2^2.
\end{align}
If $\bar r_k\ge 2/\eta+\delta$, then $1-\eta\bar r_k\le -(1+\eta\delta)$, so
\begin{align}
\langle d_{k+1},d_k\rangle \le -(1+\eta\delta)\|d_k\|_2^2.
\end{align}
Cauchy--Schwarz then gives $\|d_{k+1}\|_2\ge (1+\eta\delta)\|d_k\|_2$.
Iterating from $j=s$ to $j=t-1$ yields $\|d_t\|_2\ge (1+\eta\delta)^{t-s}\|d_s\|_2$.
\end{proof}

\begin{proof}[Proof of \cref{thm:robust_oscillatory}]
Unrolling the recursion gives
\begin{align}
x_T=-\eta\sum_{s=0}^{T-1} a_s u_s,
\qquad
a_s\triangleq \prod_{r=s+1}^{T-1} m_r.
\end{align}
To exploit the alternating signs, write $a_s=(-1)^{T-1-s}b_s$, where
\begin{align}
b_s\triangleq \prod_{r=s+1}^{T-1}|m_r|.
\end{align}
Since $|m_r|\le 1$, the sequence $b_0\le b_1\le\cdots\le b_{T-1}=1$ is nondecreasing.
Define the partial sums $A_j\triangleq \sum_{s=0}^j a_s$. The alternating signs and the monotonicity of the $b_s$ sequence give $|A_j|\le 1$ for every $j$.
Applying Abel summation:
\begin{align}
\sum_{s=0}^{T-1} a_su_s
= A_{T-1}u_{T-1}+\sum_{s=0}^{T-2}A_s(u_s-u_{s+1}).
\end{align}
Taking absolute values and using $|A_j|\le 1$ gives \eqref{eq:robust_oscillatory_bound}.
\end{proof}

%% ====================================================================
\section{Curvature Concentration under Mini-Batch SGD}\label{app:sgd}
%% ====================================================================

The forcing theorems of \cref{sec:eos} were stated for full-batch gradient descent, but the telescoping structure that drives them does not require exact gradients. This appendix extends the edge-balance identity to mini-batch SGD.

When the gradient estimate at step $k$ is $\nabla L(w_k) + \xi_k$ rather than $\nabla L(w_k)$, the one-step loss change picks up two additional terms beyond the deterministic formula. The first is a cross term between the true gradient and the noise, which vanishes in expectation under the standard martingale-difference assumption. The second is a variance term proportional to $\|\xi_k\|_2^2$, which persists in expectation and shifts the balance identity by the cumulative squared noise magnitude. The telescoping structure otherwise survives intact.

\begin{theorem}[Stochastic curvature concentration]\label{thm:sgd_edge_balance}
Let
\begin{align}
w_{k+1}=w_k-\eta\bigl(\nabla L(w_k)+\xi_k\bigr),
\end{align}
where $\xi_k\in\mathbb R^d$ is an arbitrary noise sequence. Set
$s_k\triangleq w_{k+1}-w_k$, and define $\bar H_k$, $\widetilde H_k$, and
\begin{align}
\widetilde r_k\triangleq \frac{s_k^\top\widetilde H_k s_k}{\|s_k\|_2^2}
\qquad (s_k\neq 0)
\end{align}
along the stochastic step segment $w_k+\tau s_k$ exactly as in \cref{thm:propagator_loss}.
Then:

\textup{(i)} The step increments satisfy the exact forced propagator
\begin{align}\label{eq:sgd_propagator_exact}
s_{k+1}=(I-\eta\bar H_k)s_k-\eta(\xi_{k+1}-\xi_k).
\end{align}

\textup{(ii)} The one-step loss change satisfies
\begin{align}\label{eq:sgd_loss_balance_exact}
L(w_{k+1})-L(w_k)
=
-\frac{\|s_k\|_2^2}{2\eta}\Bigl(2-\eta\widetilde r_k\Bigr)
+\eta\langle \nabla L(w_k),\xi_k\rangle
+\eta\|\xi_k\|_2^2.
\end{align}

\textup{(iii)} Summing over $k=0,\dots,K-1$ gives the stochastic telescoping identity
\begin{align}\label{eq:sgd_edge_balance_exact}
\sum_{k=0}^{K-1}\|s_k\|_2^2\Bigl(\frac{2}{\eta}-\widetilde r_k\Bigr)
=
2\bigl(L(w_0)-L(w_K)\bigr)
+2\eta\sum_{k=0}^{K-1}\langle \nabla L(w_k),\xi_k\rangle
+2\eta\sum_{k=0}^{K-1}\|\xi_k\|_2^2.
\end{align}
If $\{\xi_k\}$ is a square-integrable martingale-difference sequence with respect to the SGD filtration, then
\begin{align}\label{eq:sgd_edge_balance_expectation}
\mathbb E\!\left[\sum_{k=0}^{K-1}\|s_k\|_2^2\Bigl(\frac{2}{\eta}-\widetilde r_k\Bigr)\right]
=
2\mathbb E\bigl[L(w_0)-L(w_K)\bigr]
+2\eta\sum_{k=0}^{K-1}\mathbb E\|\xi_k\|_2^2.
\end{align}
\end{theorem}

\begin{proof}
For \eqref{eq:sgd_propagator_exact}, write
\begin{align}
s_{k+1}
&=-\eta\bigl(\nabla L(w_{k+1})+\xi_{k+1}\bigr)\\
&=-\eta\bigl(\nabla L(w_k)+\xi_k\bigr)-\eta\bigl(\nabla L(w_{k+1})-\nabla L(w_k)\bigr)-\eta(\xi_{k+1}-\xi_k)\\
&=s_k-\eta\bar H_k s_k-\eta(\xi_{k+1}-\xi_k).
\end{align}

For \eqref{eq:sgd_loss_balance_exact}, the Taylor expansion with integral remainder and the substitution $s_k=-\eta(\nabla L(w_k)+\xi_k)$ give
\begin{align}
L(w_{k+1})-L(w_k)
&=\nabla L(w_k)^\top s_k+\frac12 s_k^\top\widetilde H_k s_k,\\
\nabla L(w_k)^\top s_k
&=-\eta\|\nabla L(w_k)\|_2^2-\eta\langle \nabla L(w_k),\xi_k\rangle\\
&=-\frac1\eta\|s_k\|_2^2+\eta\langle \nabla L(w_k),\xi_k\rangle+\eta\|\xi_k\|_2^2,
\end{align}
where the last line uses $\|s_k\|_2^2=\eta^2\|\nabla L(w_k)+\xi_k\|_2^2$.
Substituting yields \eqref{eq:sgd_loss_balance_exact}; summing yields \eqref{eq:sgd_edge_balance_exact}. The expectation identity follows from $\mathbb E[\langle \nabla L(w_k),\xi_k\rangle\mid\mathcal F_k]=0$.
\end{proof}

%% ====================================================================
\section{Discrete-Time Kelvin--Voigt Dynamics}\label{sec:discrete}
%% ====================================================================

The step recurrence $d_{k+1} = (I - \eta \bar H_k) d_k$ from \cref{thm:propagator_loss} governs how the step size of a single trajectory evolves. A natural next question is: how does the \emph{difference} between two trajectories evolve? This is the question of algorithmic stability~\citep{bousquet2002stability,hardt2016train}, and the edge coupling provides a natural framework for answering it.

The idea is to replace the homogeneous coupling $\mathcal{A}_\eta$, which uses the same loss at both iterates, with a heterogeneous variant that assigns different losses to the two positions:
\begin{align}\label{eq:het_euler_action}
\mathcal A_\eta^{S,S'}(x,y)
\;\triangleq\;
L_S(x)+L_{S'}(y)-\frac{1}{2\eta}\|x-y\|^2.
\end{align}
Setting $\nabla_x \mathcal{A}_\eta^{S,S'} = 0$ recovers the gradient-descent step on dataset $S$, while $\nabla_y \mathcal{A}_\eta^{S,S'} = 0$ recovers the step on dataset $S'$. The parameter deviation $\delta_k = w_k - w_k'$ between the two trajectories then satisfies a forced linear recurrence whose stability boundary is again $2/\eta$.

The rest of this appendix makes this precise. We first define the quantities that measure the deviation (\cref{def:discrete_strain_stress}), then derive the recurrence they satisfy (\cref{thm:discrete_kv}), and finally solve it via a discrete propagator (\cref{thm:discrete_strain_integral}).

\subsection{Discrete Algorithmic Strain and Data Stress}\label{subsec:discrete_strain}

When two copies of gradient descent are run on different datasets from a common initialization, the parameter deviation grows due to two distinct effects: the curvature of the loss landscape amplifies existing displacement, and the gradient mismatch between the two datasets injects new displacement at each step. To disentangle these, we introduce the following quantities. Consider gradient descent with step size $\eta > 0$ applied to two datasets:
\begin{align}\label{eq:gd_updates}
w_{k+1} = w_k - \eta \nabla L_S(w_k), \qquad w'_{k+1} = w'_k - \eta \nabla L_{S'}(w'_k),
\end{align}
with common initialization $w_0 = w'_0$.

\begin{definition}[Discrete Strain and Stress]\label{def:discrete_strain_stress}
The discrete algorithmic strain and data stress are
\begin{align}\label{eq:discrete_strain}
\delta_k \triangleq w_k - w'_k, \qquad f_k \triangleq \nabla L_S(w'_k) - \nabla L_{S'}(w'_k).
\end{align}
The discrete stability matrix is the segment-averaged Hessian
\begin{align}\label{eq:discrete_stability}
A_k \triangleq \int_0^1 \nabla^2 L_S\bigl(w'_k + \tau \delta_k\bigr)\, d\tau.
\end{align}
\end{definition}

\subsection{The Discrete Kelvin--Voigt Equation}\label{subsec:discrete_kv}

The strain $\delta_k$ evolves through two contributions at each step: the gradient of $L_S$ at $w_k$ versus at $w_k'$ (which depends on both the curvature and the current displacement), and the gradient mismatch $f_k$ between the two datasets at the reference point $w_k'$. The mean value theorem along the segment from $w_k'$ to $w_k$ absorbs the displacement dependence into the stability matrix $A_k$, separating curvature amplification from data-driven forcing.

\begin{theorem}[Discrete Kelvin--Voigt Variational Equation]\label{thm:discrete_kv}
The discrete algorithmic strain satisfies the recurrence
\begin{align}\label{eq:discrete_kv}
\delta_{k+1} = (I - \eta A_k)\,\delta_k - \eta f_k, \qquad \delta_0 = 0.
\end{align}
This is precisely the explicit Euler discretization with time step $\eta$ of the continuous Kelvin--Voigt equation \cref{eq:kelvin_voigt}.
\end{theorem}

\begin{proof}
Subtracting the GD updates \cref{eq:gd_updates}:
\begin{align}
\delta_{k+1} &= w_{k+1} - w'_{k+1} = \delta_k - \eta\bigl(\nabla L_S(w_k) - \nabla L_{S'}(w'_k)\bigr).
\end{align}
Adding and subtracting $\nabla L_S(w'_k)$:
\begin{align}
\delta_{k+1} &= \delta_k - \eta\bigl(\nabla L_S(w_k) - \nabla L_S(w'_k)\bigr) - \eta\underbrace{\bigl(\nabla L_S(w'_k) - \nabla L_{S'}(w'_k)\bigr)}_{= f_k}.
\end{align}
Applying the mean value theorem to the first bracketed term:
\begin{align}
\nabla L_S(w_k) - \nabla L_S(w'_k) = \Bigl(\int_0^1 \nabla^2 L_S(w'_k + \tau\delta_k)\,d\tau\Bigr)\delta_k = A_k\,\delta_k.
\end{align}
Substituting yields $\delta_{k+1} = (I - \eta A_k)\delta_k - \eta f_k$.
\end{proof}

\begin{remark}[From Continuous to Discrete]\label{rem:continuous_to_discrete}
Forward Euler discretization of $\partial_t \delta + A(t)\delta = -f(t)$ gives $\delta_{k+1} = (I - \eta A_k)\delta_k - \eta f_k$. \Cref{thm:discrete_kv} shows that the discrete gradient-descent strain satisfies this equation with no approximation, so the stability boundaries of explicit Euler are inherited by the actual dynamics.
\end{remark}

\subsection{The Discrete Propagator and Strain Integral}\label{subsec:discrete_propagator}

The recurrence in \cref{thm:discrete_kv} is a forced linear system, and such systems can be solved by variation of constants. The idea is to first understand the homogeneous part (how a displacement at step $s$ is transported forward to step $k$ by the curvature alone), and then sum up the contributions of each past stress $f_s$ after it has been transported forward. The object that encodes this forward transport is the discrete propagator.

\begin{definition}[Discrete Propagator]\label{def:discrete_propagator}
The discrete propagator $\mathcal{T}[k, s]$ for the system \cref{eq:discrete_kv} is
\begin{align}\label{eq:discrete_propagator}
\mathcal{T}[k, s] \triangleq \begin{cases}
(I - \eta A_{k-1})(I - \eta A_{k-2}) \cdots (I - \eta A_s), & k > s, \\
I, & k = s.
\end{cases}
\end{align}
\end{definition}

\begin{theorem}[Discrete Strain Propagation Formula]\label{thm:discrete_strain_integral}
The discrete algorithmic strain admits the representation
\begin{align}\label{eq:discrete_strain_integral}
\delta_k = -\eta \sum_{s=0}^{k-1} \mathcal{T}[k, s+1]\, f_s.
\end{align}
\end{theorem}

\begin{proof}
By induction on $k$. For $k = 0$, $\delta_0 = 0$ and the sum is empty. Assuming the formula holds for $k$, the recurrence \cref{eq:discrete_kv} gives
\begin{align}
\delta_{k+1} &= (I - \eta A_k)\,\delta_k - \eta f_k \\
&= (I - \eta A_k)\Bigl(-\eta \sum_{s=0}^{k-1} \mathcal{T}[k, s+1]\, f_s\Bigr) - \eta f_k \\
&= -\eta \sum_{s=0}^{k-1} \underbrace{(I - \eta A_k)\mathcal{T}[k, s+1]}_{= \mathcal{T}[k+1, s+1]} f_s - \eta \underbrace{I}_{= \mathcal{T}[k+1, k+1]} f_k \\
&= -\eta \sum_{s=0}^{k} \mathcal{T}[k+1, s+1]\, f_s,
\end{align}
which is the formula at $k+1$.
\end{proof}

%% ====================================================================
\section{Continuous-Time Kelvin--Voigt Framework}\label{sec:continuous}
%% ====================================================================

The Edge of Stability is inherently discrete~\citep{cohen2021gradient}, but developing the framework first in continuous time makes the qualitative change introduced by discretization transparent. In that setting, the difference between two nearby training trajectories satisfies a first-order linear variational equation, which we interpret as a Kelvin-Voigt-type constitutive law for algorithmic strain.

In continuous time (gradient flow), positive curvature is always stabilizing: larger eigenvalues of the Hessian cause faster contraction of perturbations, and only negative eigenvalues drive instability. Discretization breaks this monotonicity. With a finite step size $\eta$, the stability window becomes $[0, 2/\eta]$: eigenvalues above $2/\eta$ flip the sign of the one-step multiplier past $-1$ and cause oscillatory instability, even though they are positive. It is this upper boundary that gives rise to the Edge of Stability. The continuous-time derivations in this appendix parallel those of \cref{sec:discrete}, with the stability matrix $A(t)$ now defined as the path-averaged Hessian between two gradient-flow trajectories and the stress $f(t)$ as the gradient mismatch at the reference trajectory.

\subsection{Algorithmic Strain and Data Stress}

Consider two trajectories $w_S(t)$ and $w_{S'}(t)$, initialized at the same point $w_0$ and evolving under gradient flow on datasets $S$ and $S'$:
\begin{align}\label{eq:gradient_flow}
\partial_t w_S(t) = -\nabla L_S(w_S(t)), \qquad \partial_t w_{S'}(t) = -\nabla L_{S'}(w_{S'}(t)).
\end{align}
\begin{definition}[Algorithmic Strain and Data Stress]\label{def:strain_stress}
Let $w_S, w_{S'} : [0, T] \to \mathbb{R}^d$ solve \cref{eq:gradient_flow} with $w_S(0) = w_{S'}(0) = w_0$. The algorithmic strain and data stress are
\begin{align}
\delta(t) &\triangleq w_S(t) - w_{S'}(t), \label{eq:strain_def}\\
f(t) &\triangleq \nabla L_S(w_{S'}(t)) - \nabla L_{S'}(w_{S'}(t)). \label{eq:stress_def}
\end{align}
The strain measures the displacement between trajectories; the stress isolates the direct effect of the dataset change from the indirect effect of parameter displacement.
\end{definition}

\subsection{The Kelvin--Voigt Constitutive Equation}

The strain $\delta(t)$ satisfies a first-order linear ODE that arises by subtracting the two gradient flow equations. The gradient difference between the same loss $L_S$ evaluated at two nearby points is handled by the mean value theorem, which absorbs the displacement dependence into the stability matrix $A(t)$ and leaves the stress $f(t)$ as a pure forcing term.

\begin{definition}[Stability Matrix]\label{def:stability_matrix}
The stability matrix is the integral mean of the Hessian along the segment connecting the two trajectories:
\begin{align}\label{eq:stability_matrix}
A(t) \triangleq \int_0^1 \nabla^2 L_S\bigl(w_{S'}(t) + \tau(w_S(t) - w_{S'}(t))\bigr) \, d\tau.
\end{align}
\end{definition}

\begin{theorem}[Kelvin--Voigt Variational Equation]\label{thm:kelvin_voigt}
The algorithmic strain $\delta(t)$ satisfies the variational equation
\begin{align}\label{eq:kelvin_voigt}
\partial_t \delta(t) + A(t) \delta(t) = -f(t), \quad \delta(0) = 0.
\end{align}
This has the form of the constitutive equation for an anisotropic Kelvin--Voigt viscoelastic material, with $A(t)$ as the elastic modulus and $f(t)$ as the applied stress. The equation is state-dependent: $A(t)$ depends on $\delta(t)$ through the mean-value integration path.
\end{theorem}

\begin{proof}
Subtracting the gradient flow equations and decomposing gives
\begin{align}
\partial_t \delta &= -\nabla L_S(w_S) + \nabla L_{S'}(w_{S'}) \nonumber \\
&= -\nabla L_S(w_S) + \nabla L_S(w_{S'}) - \nabla L_S(w_{S'}) + \nabla L_{S'}(w_{S'}) \nonumber \\
&= -\bigl[\nabla L_S(w_S) - \nabla L_S(w_{S'})\bigr] - \underbrace{\bigl[\nabla L_S(w_{S'}) - \nabla L_{S'}(w_{S'})\bigr]}_{= f(t)}.
\end{align}
The first bracket is the gradient difference of the same objective at two points. By the mean value theorem for $C^2$ functions,
\begin{align}
\nabla L_S(w_S) - \nabla L_S(w_{S'}) = \Bigl(\int_0^1 \nabla^2 L_S\bigl(w_{S'} + \tau\,\delta\bigr)\,d\tau\Bigr)\delta = A(t)\,\delta(t).
\end{align}
Substituting yields 
\begin{align}
\partial_t \delta + A(t)\delta = -f(t). 
\end{align}
The initial condition $\delta(0)=0$ is immediate from $w_S(0) = w_{S'}(0)$.
\end{proof}

\begin{remark}[Scope of the Kelvin--Voigt Correspondence]\label{rem:kv_scope}
The classical Kelvin--Voigt model assumes $A(t) \succ 0$, so all perturbations decay. Here $A(t)$ can be indefinite during saddle-point traversal: negative eigenvalues drive strain growth, while the viscous term $\partial_t\delta$ provides dissipative competition. Instability requires sustained negative curvature, quantified by $\int \alpha_-(t)\,dt$ in \cref{thm:stability_bound}. In discrete time (\cref{sec:discrete}), positive curvature also has a stability boundary: eigenvalues exceeding $2/\eta$ cause oscillatory instability, a phenomenon absent in continuous time.
\end{remark}

The solution of the Kelvin--Voigt equation is expressed through the propagator $\mathcal{T}(t,s)$, which maps the state at time $s$ to the state at time $t$ under the homogeneous dynamics.

\begin{definition}[Parameter-Space Propagator]\label{def:propagator}
The propagator $\mathcal{T}(t, s)$ is the unique solution to
\begin{align}\label{eq:propagator_ode}
\begin{aligned}
\frac{\partial}{\partial t} \mathcal{T}(t, s) &= -A(t) \mathcal{T}(t, s), \quad \text{for } t \ge s, \\
\mathcal{T}(s, s) &= I_d.
\end{aligned}
\end{align}
The propagator satisfies the semigroup property $\mathcal{T}(t, r) = \mathcal{T}(t, s) \mathcal{T}(s, r)$ for any $t \ge s \ge r$, and the inverse relation $\mathcal{T}(t, s)^{-1} = \mathcal{T}(s, t)$.
\end{definition}

The variation-of-constants formula then expresses $\delta(t)$ as a convolution of past stresses with the propagator.

\begin{theorem}[Strain Propagation Integral]\label{thm:strain_integral}
The algorithmic strain $\delta(t)$ admits the integral representation
\begin{align}\label{eq:strain_integral}
\delta(t) = -\int_{0}^{t} \mathcal{T}(t, s) f(s) \, ds.
\end{align}
\end{theorem}

\begin{proof}
Differentiating \cref{eq:strain_integral} via the Leibniz rule:
\begin{align}
\frac{d}{dt} \delta(t) 
&= \frac{d}{dt} \left( -\int_{0}^{t} \mathcal{T}(t, s) f(s) \, ds \right) \\
&= -\mathcal{T}(t, t) f(t) - \int_{0}^{t} \frac{\partial}{\partial t} \mathcal{T}(t, s) f(s) \, ds.
\end{align}
The boundary term is $-f(t)$. Substituting $\partial_t\mathcal{T}(t,s) = -A(t)\mathcal{T}(t,s)$ into the integral term:
\begin{align}
- \int_{0}^{t} \frac{\partial}{\partial t} \mathcal{T}(t, s) f(s) \, ds 
&= - \int_{0}^{t} \bigl( -A(t) \mathcal{T}(t, s) \bigr) f(s) \, ds \\
&= A(t) \int_{0}^{t} \mathcal{T}(t, s) f(s) \, ds \\
&= -A(t) \delta(t).
\end{align}
Combining gives $\dot\delta = -f(t) - A(t)\delta(t)$, which is \cref{eq:kelvin_voigt}. The initial condition $\delta(0) = 0$ follows from the vanishing domain at $t=0$, and uniqueness from Picard--Lindel\"of.
\end{proof}

\subsection{KV Stability: The Role of Curvature}

With the strain integral in hand (\cref{thm:strain_integral}), bounding $\|\delta(t)\|$ reduces to bounding the propagator norm $\|\mathcal{T}(t,s)\|_{\mathrm{op}}$. A naive Gr\"onwall bound would use $\|A(\tau)\|_{\mathrm{op}}$, but this is unnecessarily pessimistic: it treats large positive eigenvalues as destabilizing, when in continuous time they cause faster contraction. The correct bound uses only the negative part of the spectrum. This is the fundamental difference between continuous and discrete time: in continuous time, instability requires negative curvature; in discrete time (\cref{sec:discrete}), curvature above $2/\eta$ is equally destabilizing.

\begin{definition}[Instantaneous Negative Curvature]\label{def:expansivity}
The instantaneous negative curvature is
\begin{align}\label{eq:expansivity}
\alpha_-(t) \triangleq \max \left\{ 0, \, -\lambda_{\min} \left( \frac{A(t) + A(t)^\top}{2} \right) \right\}.
\end{align}
When $A(t)$ is symmetric, this simplifies to $\alpha_-(t) = \max\{0, -\lambda_{\min}(A(t))\}$.
\end{definition}

\begin{theorem}[Effective Stability Bound]\label{thm:stability_bound}
For any $0 \le s \le t$, the operator norm of the propagator is controlled by the cumulative negative curvature:
\begin{align}\label{eq:propagator_bound}
\|\mathcal{T}(t, s)\|_{\mathrm{op}} \le \exp \left( \int_{s}^{t} \alpha_-(\tau) \, d\tau \right).
\end{align}
Consequently, the algorithmic strain satisfies the bound
\begin{align}\label{eq:strain_bound}
\|\delta(t)\|_2 \le \int_{0}^{t} \exp \left( \int_{s}^{t} \alpha_-(\tau) \, d\tau \right) \|f(s)\|_2 \, ds.
\end{align}
In particular, if the loss landscape is locally convex along the trajectory (so that $\alpha_-(t) = 0$ for all $t$), then the propagator is a contraction and $\|\delta(t)\|_2 \le \int_0^t \|f(s)\|_2 \, ds$.
\end{theorem}

\begin{proof}
Fix a unit vector $v$ and set $u(\tau) = \mathcal{T}(\tau,s)v$, $\psi(\tau) = \|u(\tau)\|_2^2$. The propagator equation $\partial_\tau u = -A(\tau)u$ gives
\begin{align}
\frac{d\psi}{d\tau} = -2\,u(\tau)^\top A(\tau)\,u(\tau).
\end{align}
Writing $A = H + S$ with $H$ symmetric and $S$ skew-symmetric, the skew part drops out of the quadratic form, so
\begin{align}
u^\top A\,u = u^\top H\,u \ge \lambda_{\min}(H)\,\|u\|_2^2,
\end{align}
and therefore $\dot\psi \le 2\alpha_-(\tau)\,\psi(\tau)$. Gr\"onwall's inequality with $\psi(s)=1$ gives
\begin{align}
\psi(t) \le \exp\!\Bigl(2\int_s^t \alpha_-(\tau)\,d\tau\Bigr).
\end{align}
Taking square roots and supremizing over unit $v$ yields \cref{eq:propagator_bound}. The strain bound follows from the triangle inequality applied to \cref{eq:strain_integral}:
\begin{align}
\|\delta(t)\|_2
\le \int_0^t \|\mathcal{T}(t,s)\|_{\mathrm{op}}\,\|f(s)\|_2\,ds
\le \int_0^t \exp\!\Bigl(\int_s^t \alpha_-(\tau)\,d\tau\Bigr)\|f(s)\|_2\,ds.
\end{align}
\end{proof}

In convex regions ($A(t) \succeq 0$), the propagator is a contraction and instability arises only from non-convex portions of the trajectory. This is where continuous and discrete time diverge. In the discrete setting of \cref{sec:discrete}, explicit time-stepping introduces a finite stability window $[0,2/\eta]$: eigenvalues above $2/\eta$ cause oscillatory instability even though they are positive. It is this upper boundary that gives rise to the Edge of Stability.

%% ====================================================================
\section{Transverse edge normal form for two-layer linear networks}\label{app:linear_Q}
%% ====================================================================

This appendix proves \cref{prop:linear_Q} by computing $\mathcal P$ for the two-layer linear network loss
\begin{align}
L_h(W_1, W_2) = \tfrac{1}{2}\|W_2 W_1 - M\|_F^2
\end{align}
at a balanced global minimizer. Because the network is linear, all derivatives of $L_h$ can be computed in closed form, which makes this a concrete test case for the general bifurcation theory of \cref{thm:intrinsic_edge_potential,cor:generic_edge_branches}.

The argument proceeds in four steps. First, we identify a canonical minimizer using the SVD of the target matrix $M$. Second, we compute the Hessian kernel, which consists of reparametrization symmetries that do not change the product $W_2 W_1$. Third, we restrict the edge coupling theory to the orthogonal complement of this kernel (the normal space $\mathcal{N}$), and show that this restriction is independent of the hidden width $h$. Fourth, we evaluate the branch form $\mathcal Q$ along the leading eigenvector and find that it is negative, giving a supercritical bifurcation.

\paragraph{Canonical balanced minimizer.}
The starting point is the SVD of the target matrix. Let
\begin{align}
M = U_r \Sigma V_r^\top,
\qquad
\Sigma=\operatorname{diag}(\sigma_1,\dots,\sigma_r),
\qquad
\sigma_1\ge\cdots\ge\sigma_r>0.
\end{align}
Every balanced global minimizer has the form
\begin{align}
\bar W_2=U_r\Sigma^{1/2}R^\top,
\qquad
\bar W_1=R\Sigma^{1/2}V_r^\top,
\qquad
R\in\mathbb R^{h\times r},
\quad
R^\top R=I_r .
\end{align}
Since the loss and the Frobenius metric are invariant under orthogonal changes of output, input, and hidden coordinates, we may choose coordinates so that
\begin{align}
M=
\begin{pmatrix}
\Sigma & 0\\
0 & 0
\end{pmatrix},
\qquad
\bar W_1^{(h)}=
\begin{pmatrix}
\Sigma^{1/2} & 0\\
0 & 0
\end{pmatrix},
\qquad
\bar W_2^{(h)}=
\begin{pmatrix}
\Sigma^{1/2} & 0\\
0 & 0
\end{pmatrix}.
\end{align}

\paragraph{Kernel and normal slice.}
The Hessian at $\bar w$ has a nontrivial kernel because many different factorizations $W_2 W_1 = M$ achieve the same product and hence the same loss. These are the reparametrization symmetries of the network. The center-reduction \cref{thm:intrinsic_edge_potential} requires a nondegenerate Hessian, so we must identify the kernel $\mathcal{K}$ and restrict to its orthogonal complement $\mathcal{N}$, where the curvature is nonzero and the bifurcation analysis applies. Write perturbations as
\begin{align}
\delta W_1=
\begin{pmatrix}
A & B\\
C & D
\end{pmatrix},
\qquad
\delta W_2=
\begin{pmatrix}
E & F\\
G & H
\end{pmatrix},
\end{align}
where $A \in \mathbb{R}^{r \times r}$, $B \in \mathbb{R}^{r \times (d-r)}$, etc., matching the block structure of $\bar W_1^{(h)}$. Define the product map $F_h(W_1,W_2)\triangleq W_2W_1$. At a global minimizer the residual vanishes, so the Hessian factors as
\begin{align}
\nabla^2L_h(\bar w)=DF_h(\bar w)^\ast DF_h(\bar w).
\end{align}
A direct computation gives the linearization and its kernel:
\begin{align}
DF_h(\bar w)[\delta W_1,\delta W_2]
&=
\bar W_2\delta W_1+\delta W_2\bar W_1
=
\begin{pmatrix}
\Sigma^{1/2}A+E\Sigma^{1/2} & \Sigma^{1/2}B\\
G\Sigma^{1/2} & 0
\end{pmatrix},\\
\mathcal K&=\ker\nabla^2L_h(\bar w)=\ker DF_h(\bar w) \\
&=
\left\{
\begin{array}{l}
B=0,\quad G=0,\quad \Sigma^{1/2}A+E\Sigma^{1/2}=0,\\[0.3ex]
C,D,F,H\ \text{arbitrary}
\end{array}
\right\}.
\end{align}

The image of the linearization determines what directions in the output space the factorization can explore:
\begin{align}
\operatorname{im} DF_h(\bar w)
=
\left\{
\begin{pmatrix}
Z & X\\
Y & 0
\end{pmatrix}
:
Z\in\mathbb R^{r\times r},\
X\in\mathbb R^{r\times(d-r)},\
Y\in\mathbb R^{(p-r)\times r}
\right\},
\end{align}
which is the tangent space $T_M\mathcal R_r$ of the rank-$r$ matrix manifold at $M$. Therefore $DF_h(\bar w)$ has rank $r(p+d-r)$, so by the constant-rank theorem the minimum set
\begin{align}
\mathcal M_h(M)\triangleq \{(W_1,W_2):W_2W_1=M\}
\end{align}
is a smooth manifold near $\bar w$ with tangent space $T_{\bar w}\mathcal M_h(M)=\mathcal K$. Thus $\bar w$ is Morse--Bott.

The kernel $\mathcal{K}$ consists of flat directions that do not affect the loss. To apply the edge coupling theory, we restrict to the orthogonal complement $\mathcal{N} = \mathcal{K}^\perp$, which captures the directions with nonzero curvature. Define
\begin{align}
T(A,E)\triangleq \Sigma^{1/2}A+E\Sigma^{1/2}.
\end{align}
Its adjoint is
\begin{align}
T^\ast(Y)=\bigl(\Sigma^{1/2}Y,\;Y\Sigma^{1/2}\bigr),
\end{align}
so
\begin{align}
(\ker T)^\perp=\operatorname{ran} T^\ast
=
\left\{(\Sigma^{1/2}Y,\;Y\Sigma^{1/2}) : Y\in\mathbb R^{r\times r}\right\}.
\end{align}
Combining the $(\ker T)^\perp$ constraint on the $(A,E)$-block with unconstrained $B$ and $G$ blocks (and the zero blocks from the kernel conditions $C=D=F=H=0$) gives the full normal space:
\begin{align}
\mathcal N
=
\left\{
\delta W_1=
\begin{pmatrix}
\Sigma^{1/2}Y & B\\
0 & 0
\end{pmatrix},
\qquad
\delta W_2=
\begin{pmatrix}
Y\Sigma^{1/2} & 0\\
G & 0
\end{pmatrix}
:
Y,B,G
\right\}.
\end{align}

\paragraph{Exact width-invariance of the restricted loss.}
The central observation is that the restricted loss on $\mathcal{N}$ is independent of the hidden width $h$. Increasing $h$ beyond the rank $r$ adds only flat directions in $\mathcal{K}$, so the loss restricted to $\mathcal{N}$ is the same for any $h \ge r$. We now establish this by constructing an explicit isometry between the normal slices at different widths. For $a=(Y,B,G)$, define
\begin{align}
T_h(a)\triangleq
\left(
\begin{pmatrix}
\Sigma^{1/2}Y & B\\
0 & 0
\end{pmatrix},
\begin{pmatrix}
Y\Sigma^{1/2} & 0\\
G & 0
\end{pmatrix}
\right)\in\mathcal N.
\end{align}
The linearization and quadratic residual along this normal slice are
\begin{align}
P_h(a)
&\triangleq
DF_h(\bar w)T_h(a)
=
\begin{pmatrix}
\Sigma Y+Y\Sigma & \Sigma^{1/2}B\\
G\Sigma^{1/2} & 0
\end{pmatrix},\\
Q_h(a)
&\triangleq
\delta W_2\,\delta W_1
=
\begin{pmatrix}
Y\Sigma Y & Y\Sigma^{1/2}B\\
G\Sigma^{1/2}Y & GB
\end{pmatrix},
\end{align}
so the restricted loss takes the form
\begin{align}
L_h(\bar w+T_h(a))
=
\frac12\|P_h(a)+Q_h(a)\|_F^2.
\end{align}

The key point is that neither $P_h(a)$ nor $Q_h(a)$ depends on $h$: the extra rows and columns added by increasing the hidden width are all zero. To make this formal, consider the minimal-width model ($h = r$) at the balanced minimizer
\begin{align}
\bar W_1^{(r)}=(\Sigma^{1/2}\ \ 0),
\qquad
\bar W_2^{(r)}=\binom{\Sigma^{1/2}}{0},
\end{align}
with normal-slice embedding
\begin{align}
T_r(a)\triangleq
\left(
(\Sigma^{1/2}Y\ \ B),
\binom{Y\Sigma^{1/2}}{G}
\right).
\end{align}
The restricted loss at width $r$ produces the same $P_h$ and $Q_h$, so
\begin{align}
L_r(\bar w^{(r)}+T_r(a))
=
\frac12\|P_h(a)+Q_h(a)\|_F^2
=
L_h(\bar w+T_h(a)).
\end{align}
The zero-padding map $Z_h \triangleq T_h T_r^{-1}:\mathcal N^{(r)}\to \mathcal N^{(h)}$ is therefore an isometric isomorphism, and
\begin{align}
L_h(\bar w^{(h)}+Z_h\xi)=L_r(\bar w^{(r)}+\xi)
\qquad
(\xi\in\mathcal N^{(r)}).
\end{align}
The restricted losses coincide, and so do all their derivatives. Width-invariance of $\mathcal P$ and its quartic term follows.

\paragraph{Transverse spectrum.}
To determine the critical learning rate at which period-doubling first occurs, we need to know which eigenvalue of the restricted Hessian is largest. Since the residual vanishes at the minimizer, the Hessian factors as $\nabla^2L_h(\bar w)=DF_h(\bar w)^\ast DF_h(\bar w)$. The transverse Hessian quadratic form then decomposes into three orthogonal blocks corresponding to the $Y$, $B$, and $G$ components of the normal space:
\begin{align}
q(Y,B,G)
=
\|\Sigma Y+Y\Sigma\|_F^2
+
\|\Sigma^{1/2}B\|_F^2
+
\|G\Sigma^{1/2}\|_F^2.
\end{align}
Evaluating on the matrix units $E_{ij}$ in the $Y$-block,
$E_{i\beta}$ in the $B$-block, and $E_{\alpha j}$ in the $G$-block gives the eigenvalues and full spectrum:
\begin{align}
\lambda(E_{ij})&=\sigma_i+\sigma_j,
\qquad
\lambda(E_{i\beta})=\sigma_i,
\qquad
\lambda(E_{\alpha j})=\sigma_j,\\
\operatorname{spec}(H^\perp)
&=
\{\sigma_i+\sigma_j:1\le i,j\le r\}
\cup
\{\sigma_i:1\le i\le r\}^{\times(p+d-2r)},
\end{align}
independent of $h$.

\paragraph{Exact sharp-mode normal form.}
As $\eta$ increases from zero, the reduced Hessian $H^\perp - (2/\eta)I$ first becomes singular when $2/\eta$ hits the largest eigenvalue $2\sigma_1$, which occurs at $\eta_c = 1/\sigma_1$. This is the critical learning rate for the first period-doubling bifurcation. Assuming $\sigma_1 > \sigma_2$ so that this eigenvalue is simple, the corresponding unit eigenvector is
\begin{align}
u_c = T_h\!\left(\frac{E_{11}}{\sqrt{2\sigma_1}},\,0,\,0\right),
\end{align}
which corresponds to the perturbation $\delta W_1=\delta W_2=\frac1{\sqrt2}E_{11}$ on the active block with zeros elsewhere. Restricting the loss to this line and expanding gives
\begin{align}
L_h(\bar w+t u_c)
=
\frac12\Bigl(\sqrt{2\sigma_1}\,t+\frac12 t^2\Bigr)^2
=
L_h(\bar w)
+\sigma_1 t^2
+\frac{\sqrt{\sigma_1}}{\sqrt2}t^3
+\frac18 t^4.
\end{align}
Reading off the third and fourth derivatives and applying \cref{prop:quartic_jet} to this one-dimensional restricted loss yields the quartic term:\begin{align}
\mathcal Q^\perp(u_c)
=
\frac{3}{6}
-
\frac{(3\sqrt{2\sigma_1})^2}{2(2\sigma_1)}
=
-4.
\end{align}
Substituting into the reduced functional gives
\begin{align}
\Phi_{\eta}^{\perp}(tu_c)
&=
L_h(\bar w)
+\frac12\Bigl(2\sigma_1-\frac2\eta\Bigr)t^2
+\frac14\,\mathcal Q^\perp(u_c)\,t^4
+o(t^4)\\
&=
L_h(\bar w)
+\Bigl(\sigma_1-\frac1\eta\Bigr)t^2
-
t^4
+o(t^4).
\end{align}
The quartic coefficient is $-1 < 0$, so by \cref{cor:generic_edge_branches} the period-two orbit exists for $\eta > \eta_c = 1/\sigma_1$ and grows continuously from zero amplitude at $\eta_c$.
\qed

\paragraph{Experimental validation.}
\autoref{fig:bifurcation} validates the prediction for a two-layer linear network ($p = 5$, $h = 3$, $d = 10$, rank-3 target, $n = 200$ samples). At the global minimizer $\bar w$, the computed quartic term satisfies $\mathcal Q^\perp(u_c) < 0$, predicting that the period-doubling branch appears for $\eta>\eta_c$ and emerges continuously from zero at $\eta_c$. Full-batch GD at learning rates $\eta$ near $\eta_c = 1/\sigma_1$ confirms this: the oscillation amplitude grows continuously from zero for $\eta > \eta_c$ and tracks the $\sqrt{\eta - \eta_c}$ scaling of \cref{cor:generic_edge_branches}, with oscillations concentrated along~$u_c$.

% \newpage
% \input{checklist.tex}

\end{document}